\documentclass[twoside,11pt]{article}
\usepackage{microtype}
\usepackage{graphicx}
\usepackage{booktabs} 
\usepackage{hyperref}
\usepackage{caption}
\usepackage{amssymb}
\usepackage{subfig}
\usepackage{mathtools, amsfonts}
\usepackage{makecell}
 \usepackage{wrapfig}
\usepackage{bbm}
\usepackage{xcolor}
\usepackage{xspace}
\usepackage{fullpage}

\newcommand\nc\newcommand
\nc{\bb}[1]{\mathbb{#1}}
\renewcommand{\cal}[1]{\mathcal{#1}}

\nc\bfa{{\bf{a}}}\nc\bfA{{\boldsymbol A}}\nc\cA{{\cal A}} \nc\fA[1]{A\br*{#1}} \nc\fa[1]{a\br*{#1}}  \nc\rmA{\mathrm{A}} \nc\rma{\mathrm{a}}
\nc\bfb{{\bf{b}}}\nc\bfB{{\boldsymbol B}}\nc\cB{{\cal B}} \nc\fB[1]{B\br*{#1}} \nc\fb[1]{b\br*{#1}}  \nc\rmB{\mathrm{B}} \nc\rmb{\mathrm{b}}
\nc\bfc{{\bf{c}}}\nc\bfC{{\boldsymbol C}}\nc\cC{{\cal C}} \nc\fC[1]{C\br*{#1}} \nc\fc[1]{c\br*{#1}}  \nc\rmC{\mathrm{C}} \nc\rmc{\mathrm{c}}
\nc\bfd{{\bf{d}}}\nc\bfD{{\boldsymbol D}}\nc\cD{{\cal D}} \nc\fD[1]{D\br*{#1}} \nc\fd[1]{d\br*{#1}}  \nc\rmD{\mathrm{D}} \nc\rmd{\mathrm{d}}
\nc\bfe{{\bf{e}}}\nc\bfE{{\boldsymbol E}}\nc\cE{{\cal E}} \nc\fE[1]{E\br*{#1}} \nc\fe[1]{e\br*{#1}}  \nc\rmE{\mathrm{E}} \nc\rme{\mathrm{e}}
\nc\bff{{\bf{f}}}\nc\bfF{{\boldsymbol F}}\nc\cF{{\cal F}} \nc\fF[1]{F\br*{#1}} \nc\ff[1]{f\br*{#1}}  \nc\rmF{\mathrm{F}} \nc\rmf{\mathrm{f}}
\nc\bfg{{\bf{g}}}\nc\bfG{{\boldsymbol G}}\nc\cG{{\cal G}} \nc\fG[1]{G\br*{#1}} \nc\fg[1]{g\br*{#1}}  \nc\rmG{\mathrm{G}} \nc\rmg{\mathrm{g}}
\nc\bfh{{\bf{h}}}\nc\bfH{{\boldsymbol H}}\nc\cH{{\cal H}} \nc\fH[1]{H\br*{#1}} \nc\fh[1]{h\br*{#1}}  \nc\rmH{\mathrm{H}} \nc\rmh{\mathrm{h}}
\nc\bfi{{\bf{i}}}\nc\bfI{{\boldsymbol I}}\nc\cI{{\cal I}} \nc\fI[1]{I\br*{#1}} \nc\rmI{\mathrm{I}} \nc\rmi{\mathrm{i}}
\nc\bfj{{\bf{j}}}\nc\bfJ{{\boldsymbol J}}\nc\cJ{{\cal J}} \nc\fJ[1]{J\br*{#1}} \nc\fj[1]{j\br*{#1}} \nc\rmJ{\mathrm{J}} \nc\rmj{\mathrm{j}}
\nc\bfk{{\bf{k}}}\nc\bfK{{\boldsymbol K}}\nc\cK{{\cal K}} \nc\fK[1]{K\br*{#1}} \nc\fk[1]{k\br*{#1}} \nc\rmK{\mathrm{K}} \nc\rmk{\mathrm{k}}
\nc\bfl{{\bf{l}}}\nc\bfL{{\boldsymbol L}}\nc\cL{{\cal L}} \nc\fL[1]{L\br*{#1}} \nc\fl[1]{l\br*{#1}} \nc\rmL{\mathrm{L}} \nc\rml{\mathrm{l}}
\nc\bfm{{\bf{m}}}\nc\bfM{{\boldsymbol M}}\nc\cM{{\cal M}} \nc\fM[1]{M\br*{#1}} \nc\fm[1]{m\br*{#1}} \nc\rmM{\mathrm{M}} \nc\rmm{\mathrm{m}}
\nc\bfn{{\bf{n}}}\nc\bfN{{\boldsymbol N}}\nc\cN{{\cal N}} \nc\fN[1]{N\br*{#1}} \nc\fn[1]{n\br*{#1}} \nc\rmN{\mathrm{N}} \nc\rmn{\mathrm{n}}
\nc\bfo{{\bf{o}}}\nc\bfO{{\boldsymbol O}}\nc\cO{{\cal O}} \nc\fO[1]{O\br*{#1}} \nc\fo[1]{o\br*{#1}} \nc\rmO{\mathrm{O}} \nc\rmo{\mathrm{o}}
\nc\bfp{{\bf{p}}}\nc\bfP{{\boldsymbol P}}\nc\cP{{\cal P}} \nc\fP[1]{P\br*{#1}} \nc\fp[1]{p\br*{#1}} \nc\rmP{\mathrm{P}} \nc\rmp{\mathrm{p}}
\nc\bfq{{\bf{q}}}\nc\bfQ{{\boldsymbol Q}}\nc\cQ{{\cal Q}} \nc\fQ[1]{Q\br*{#1}} \nc\fq[1]{q\br*{#1}} \nc\rmQ{\mathrm{Q}} \nc\rmq{\mathrm{q}}
\nc\bfr{{\bf{r}}}\nc\bfR{{\boldsymbol R}}\nc\cR{{\cal R}} \nc\fR[1]{R\br*{#1}} \nc\fr[1]{r\br*{#1}} \nc\rmR{\mathrm{R}} \nc\rmr{\mathrm{r}}
\nc\bfs{{\bf{s}}}\nc\bfS{{\boldsymbol S}}\nc\cS{{\cal S}} \nc\fS[1]{S\br*{#1}} \nc\fs[1]{s\br*{#1}} \nc\rmS{\mathrm{S}} \nc\rms{\mathrm{s}}
\nc\bft{{\bf{t}}}\nc\bfT{{\boldsymbol T}}\nc\cT{{\cal T}} \nc\fT[1]{T\br*{#1}} \nc\ft[1]{t\br*{#1}} \nc\rmT{\mathrm{T}} \nc\rmt{\mathrm{t}}
\nc\bfu{{\bf{u}}}\nc\bfU{{\boldsymbol U}}\nc\cU{{\cal U}} \nc\fU[1]{U\br*{#1}} \nc\fu[1]{u\br*{#1}} \nc\rmU{\mathrm{U}} \nc\rmu{\mathrm{u}}
\nc\bfv{{\bf{v}}}\nc\bfV{{\boldsymbol V}}\nc\cV{{\cal V}} \nc\fV[1]{V\br*{#1}} \nc\fv[1]{v\br*{#1}} \nc\rmV{\mathrm{V}} \nc\rmv{\mathrm{v}}
\nc\bfw{{\bf{w}}}\nc\bfW{{\boldsymbol W}}\nc\cW{{\cal W}} \nc\fW[1]{W\br*{#1}} \nc\fw[1]{w\br*{#1}} \nc\rmW{\mathrm{W}} \nc\rmw{\mathrm{w}}
\nc\bfx{{\bf{x}}}\nc\bfX{{\boldsymbol X}}\nc\cX{{\cal X}} \nc\fX[1]{X\br*{#1}} \nc\fx[1]{x\br*{#1}} \nc\rmX{\mathrm{X}} \nc\rmx{\mathrm{x}}
\nc\bfy{{\bf{y}}}\nc\bfY{{\boldsymbol Y}}\nc\cY{{\cal Y}} \nc\fY[1]{Y\br*{#1}} \nc\fy[1]{y\br*{#1}} \nc\rmY{\mathrm{Y}} \nc\rmy{\mathrm{y}}
\nc\bfz{{\bf{z}}}\nc\bfZ{{\boldsymbol Z}}\nc\cZ{{\cal Z}} \nc\fZ[1]{Z\br*{#1}} \nc\fz[1]{z\br*{#1}} \nc\rmZ{\mathrm{Z}} \nc\rmz{\mathrm{z}}

\DeclareMathOperator{\Var}{var}

\nc\defeq{\coloneqq}

\newcommand{\eps}{\epsilon}

\newcommand\R{{\mathbb R}}

\newcommand\F{{\mathbb F}}

\newtheorem{theorem}{Theorem}
\newtheorem{definition}{Definition}
\newtheorem{lemma}[theorem]{Lemma}
\newtheorem{proposition}[theorem]{Proposition}

\newtheorem{corollary*}[theorem]{Corollary}

\newtheorem{remark}{\indent Remark}

\newenvironment{proof}{\noindent{\bf Proof : \ }}{\hfill$\Box$\par\medskip}

\newenvironment{proofof}[1]{\begin{trivlist} \item {\bf Proof
#1:~~}}
  {\qed\end{trivlist}}
\renewenvironment{proofof}[1]{\par\medskip\noindent{\bf Proof of #1: \ }}{\hfill$\Box$\par\medskip}
\newcommand{\namedref}[2]{\hyperref[#2]{#1~\ref*{#2}}}

\newcommand{\E}{\mathbf{E}}

\newcommand{\Z}{\ensuremath{\mathbb{Z}}}

\newcommand{\supp}[2]{\text{\sc{supp}}(#1, #2)}
\newcommand{\p}[1]{\boldsymbol{#1}}
\newcommand{\ie}{\emph{i.e.},~}

\newcommand{\rappor}{\textsc{Rappor }}
\newcommand{\RR}{\textsc{RR }}
\newcommand{\onehot}{\textsc{1-hot }}
\newcommand{\ham}{\Delta}

\DeclareMathOperator*{\argmin}{arg\,min}

\newcommand{\bpsk}{\Phi}

\providecommand{\email}[1]{\href{mailto:#1}{\nolinkurl{#1}\xspace}}

\newif\ifnotes\notestrue
\ifnotes
\newcommand{\gnote}[1]{\textcolor{blue}{{\bf (GV:} {#1}{\bf ) }} \marginpar{\tiny\bf
             \begin{minipage}[t]{0.5in}
               \raggedright ~
                \end{minipage}}}  
\else
\newcommand{\gnote}[1]{}
\fi

\newcommand{\calm}{\mathcal{M}}

\newcommand{\calx}{\mathcal{X}}

\newcommand{\conv}{\text{\sc{conv}}}

\newcommand{\ignore}[1]{}
\newcommand{\sidenote}[1]{ \marginpar{\tiny\bf
             \begin{minipage}[t]{0.5in}
               \raggedright *
            \end{minipage}}}

\begin{document}

\title{vqSGD: Vector Quantized Stochastic Gradient Descent\footnote{This research is supported in parts by NSF awards CCF 1642658, CCF 1618512, CCF 1909046, and by UMass Center for Data Science.}}

\author{
Venkata Gandikota\thanks{Department of  Electrical Engineering and Computer Science, Syracuse, NY.
Email: \email{gandikota.venkata@gmail.com}.}
\and
Daniel Kane \thanks{Departments of Mathematics and CSE, University of California,  San Diego, CA.
Email: \email{dakane@ucsd.edu}.}
\and
Raj Kumar Maity\thanks{College of Information and Computer Sciences, University of Massachusetts, Amherst, MA. 
Email: \email{rajkmaity@cs.umass.edu}.}
\and
Arya Mazumdar\thanks{College of Information and Computer Sciences, University of Massachusetts, Amherst, MA.
E-mail: \email{arya@cs.umass.edu}}\\
}


\maketitle

\begin{abstract}
In this work, we present a family of vector quantization schemes \emph{vqSGD} (Vector-Quantized Stochastic Gradient Descent) that provide an asymptotic reduction in the communication cost with convergence guarantees in first-order distributed optimization. 
In the process we derive the following fundamental information theoretic fact: $\Theta(\frac{d}{R^2})$ bits are necessary and sufficient (up to an additive $O(\log d)$ term) to describe an unbiased estimator $\p{\hat{g}}(\p{g})$ for any $\p{g}$ in the $d$-dimensional unit sphere, under the constraint that $\|\p{\hat{g}}(\p{g})\|_2\le R$ almost surely.
In particular, we consider a randomized scheme based on the convex hull of a point set, that returns an unbiased estimator of a $d$-dimensional gradient vector with almost surely bounded norm. 
We provide multiple efficient instances of our scheme, that are near optimal, and require only $o(d)$ bits of communication at the expense of tolerable increase in error. The instances of our quantization scheme are obtained using the properties of binary error-correcting codes and provide a smooth tradeoff between the communication and the estimation error of quantization. Furthermore, we show that \emph{vqSGD} also offers automatic privacy guarantees.

\end{abstract}

\section{Introduction}
Recent surge in the volumes of available data has motivated the development of large-scale distributed learning algorithms. 
Synchronous Stochastic Gradient Descent (SGD) is one such learning algorithm widely used to train large models.  
In order to minimize the empirical loss, the SGD algorithm, in every iteration takes a small step in the negative direction of the \emph{stochastic gradient} which is an unbiased estimate of the true gradient of the loss function.

In this work, we consider the data-distributed model of distributed SGD where the data sets are partitioned across various compute nodes. In each iteration of SGD, the compute nodes send their computed local gradients to a parameter server that averages  and updates the global parameters. 
The distributed SGD model is highly scalable, however, with the exploding dimensionality of data and the increasing number of servers  (such as in a Federated learning setup \cite{konevcny2016federated}), communication becomes a \emph{bottleneck} to the efficiency and speed of learning using SGD \cite{chilimbi2014project}.

In the recent years various quantization and sparsification techniques   \cite{acharya2019distributed, qsgd, signsgd, koloskova2019decentralized, mayekar2019ratq, shalev2010trading,dme, atomo, terngrad} have been developed to alleviate the problem of communication bottleneck. Recently, \cite{kalan2019fitting} even showed the effectiveness of gradient quantization techniques for ReLU fitting. 
The goal of the quantization schemes is to efficiently compute either a low precision or a sparse unbiased estimate of the $d$-dimensional gradients.
One also requires the estimates to have a bounded second moment in order to achieve guaranteed convergence. 

Moreover, the data samples used to train the model often contain sensitive information.  
Hence, preserving privacy of the participating clients is crucial. 
Differential privacy \cite{dwork2016calibrating, privacybook} is a mathematically rigorous and  standard notion of privacy considered in both literature and in practice. Informally, it ensures that the information from the released data (e.g. the gradient estimates) cannot be used to distinguish between two \emph{neighboring} data sets.

\paragraph{Our Contribution:}
In this work, we present a family of \emph{privacy-preserving} \emph{vector-quantization} schemes that  
incur low communication costs while providing convergence guarantees. 
In particular, we provide explicit and efficient quantization schemes based on convex hull of specific structured point sets in $\R^d$ that require $O(d\log d/R^2)$ bits  in each iteration of SGD to communicate an unbiased gradient estimate that has  variance bounded above by $R^2$: this is within a $\log d$ factor of the optimal amount of communication that is necessary and sufficient for this purpose. 

At a high level, our scheme is based on the idea that any vector $\p{v} \in \R^d$ with bounded norm can be represented as a convex combination of a carefully constructed point set $C \subset \R^d$. This convex combination essentially allows us to chose a point $\p{c} \in C$ with probability proportional to its coefficient, which makes it an  unbiased estimator of $\p{v}$. 
The bound on the variance is obtained from the circumradius of the convex hull of $C$. 
Moreover, communicating the unbiased estimate is equivalent to communicating the index of $\p{c} \in C$ (according to some fixed ordering) that requires only $\log |C|$ bits.  We provide matching upper and lower bounds on this communication cost.

Large convex hulls have small variation in the coefficients of the convex combination of any two points of bounded norm. This observation allows us to obtain $\eps$-differential privacy (for any $\eps > \eps_0$), where $\eps_0$ depends on the choice of the point set. 
We also propose Randomized Response (RR) \cite{randomresponse} and RAPPOR  \cite{rappor} based mechanisms that can be used over the proposed quantization to achieve $\eps$-differential privacy (for any $\eps > 0$) with small trade-off in the variance of the estimates. 

The family of schemes described above is fairly general and can be instantiated using different structured point sets. The cardinality of the point set bounds the communication cost of the quantization scheme. 
Whereas, the diameter of the point set dictates the variance bounds and the privacy guarantees of the scheme. 

We provide a strong characterization of the point-sets that can be used for our quantization scheme. Using this characterization, we propose construction of point-sets that allow us to attain a smooth trade-off between variance and communication of the quantization scheme. 
We also propose some explicit structured point sets and show tradeoff in the various parameters guaranteed by them.  
While our randomized construction is optimal in terms of communication, the explicit schemes are within $\log d$ factor of a lower bound that we provide. 
Our results\footnote{Note that $\epsilon$ denotes the privacy parameter and $\varepsilon$ refers to the packing parameter of $\varepsilon$-nets.} are summarized in Table~\ref{tab:result-summary}.

\begin{table*}[ht]
\small	
    \centering
    \begin{tabular}{| l |c|c|c|c|}
    \hline 
    Point set  &  Error & \makecell[l]{ Communication (bits) \\ (in each iteration)} & Privacy & Efficiency \\ [1ex]
    \hline 
      \makecell[l]{ Gaussian-Sampling\\
    \small{(Theorem~\ref{thm:random-const})} for any $c > \log(d)$}& $\frac{d}{c N}$ & $N c$ & - & $O(\text{exp}(c))$ \\[1ex]
    \hline
     \makecell[l]{ Reed-Muller ($C_{RM}$)\\
    \small{(Proposition~\ref{prop:RMcode})}}& $\frac{d}{N}$ & $N \log 2d$ & - & $O(d)$ \\[1ex]
    \hline
   \makecell[l]{ Cross-polytope ($C_{cp}$)\\
    \small{(Proposition~\ref{prop:crosspolytope})}  }& $\frac{d }{N}$ & $N \log 2d$ & $\eps > O(\log d)$ & $O(d)$ \\[1ex]
    \hline    
    \makecell[l]{ Scaled $\varepsilon$-Net ($C_{net}$)\\
     \small{(Proposition~\ref{prop:epsnet})} }& $\frac{1}{N}$ & $ O_{\varepsilon}(Nd)\footnotemark$ & - & $O\left( \left(\frac{1}{\varepsilon}\right)^d \right)$ \\[1ex]
    \hline
   \makecell[l]{ Simplex ($C_S$)\\  \small{(Proposition ~\ref{prop:simplex})} }&$\frac{d^2}{N}$    &  $N \log (d+1)$ & $\eps > \log 7$ &$O(d)$ 
   \\[1ex]
   \hline
   \makecell[l]{ Hadamard ($C_H$)\\  \small{(Proposition ~\ref{prop:hadamard})} }&$\frac{d^2}{N}$    &  $N \log d$ & $\eps > \log(1+\sqrt{2})$ &  $O(d)$
   \\[1ex]
   \hline
\makecell[l]{ Cross-polytope ($C_{cp}$) + RR \\  \small{(Theorem ~\ref{thm:quantrr})} }&$\frac{d^2}{N}$    &  $N \log (2d)$ & $\eps > 0$ & $O(d)$
   \\[1ex]
   \hline   
   \makecell[l]{ Cross-polytope ($C_{cp}$) + RAPPOR \\  \small{(Theorem ~\ref{thm:quantrappor})} }&$\frac{d^2}{N}$    &  $2Nd$ & $\eps > 0$ & $O(d)$
   \\[1ex]
   \hline   
    \end{tabular}
    \caption{List of results. ($N$: number of worker nodes, $d$: dimension).}
    \label{tab:result-summary}
\end{table*}
\footnotetext{$O_{\varepsilon}$ hides terms involving $\varepsilon$}

Empirically we compare our quantization schemes to the state-of-art schemes \cite{qsgd,dme}.  
We observe that our cross-polytope vqSGD, performs equally well in practice, while providing asymptotic reduction in the communication cost. The communication cost of our schemes with the state of the art results are compared in Table~\ref{tab:nonprivateq}.

While differential privacy for gradient based algorithms \cite{abadi2016deep,shokri2015privacy} were considered earlier in literature, cpSGD
\cite{cpsgd} is the only work that considered achieving differential privacy for gradient based algorithms and simultaneously minimizing the gradient communication cost.  
The authors propose a binomial mechanism to add discrete noise to the quantized gradients to achieve communication-efficient $(\eps, \delta)$-differentially private gradient descent with convergence guarantees. 
The quantization schemes used are similar to those presented in \cite{dme} and hence  require $\Omega(d)$ bits of communication per compute node in each iteration of SGD. The parameters of the binomial noise are dictated by the required privacy guarantees which in turn controls  the communication cost. 

In this work we show that certain instantiations of our quantization schemes are $\eps$-differentially private. Note that this is a much stronger privacy notion than $(\eps, \delta)$-privacy. Moreover, we get this privacy guarantee directly from the quantization schemes and hence the communication cost remains sublinear ($\log d$) in dimension. 
This in turn, resolves the following question posed in \cite{cpsgd}:  \emph{``While the above schemes (QGSD \cite{qsgd}, DME \cite{dme}) reduce the communication cost, it is unclear what (if any) privacy guarantees they offer.''}. 

We also propose a 
Randomized Response \cite{randomresponse} based private-quantization scheme that requires $O(\log d)$ bits of communication per compute node to get an $\eps$-differential privacy in every iteration while losing a factor of $O(d)$ in convergence rate. 
Table~\ref{tab:privateeq} compares the guarantees provided by our private quantization schemes with the results of cpSGD \cite{cpsgd}. Similar to the results of \cite{cpsgd}, the privacy guarantees listed in the tables hold for \emph{each-iteration} of SGD. The privacy guarantees over multiple iterations of SGD can be obtained using advanced composition theorem \cite{dwork2010boosting}.


\paragraph{Organization:}
In Sec.~\ref{sec:rel}, we describe some other related work on communication efficiency in the federated learning setup. We start  in Sec.~\ref{sec:bkgrnd} describing the settings for our results. The vqSGD quantization scheme is presented in Sec.~\ref{sec:quant}. In Sec.~\ref{sec:constructions} we provide a handle to test whether a point-set is a valid vqSGD scheme, and prove existence of a point-set  that achieves a communication cost equal to the dimension divided by the variance, which matches a lower bound we prove. We provide a few structured deterministic constructions of point sets in this section as well. Sec.~\ref{sec:private} emphasizes the privacy component of vqSGD - and derives the privacy parameters of several vqSGD schemes. Finally, we provide some experiments to support vqSGD in Sec.~\ref{sec:expe}. 

\begin{table*}[ht]
\centering
    \begin{tabular}{|l|c|c|}
    \hline 
    Method &  Error & \makecell[l]{ Communication (bits) \\ (in each iteration)} \\ 
    \hline 
 \makecell[l]{ QSGD \cite{qsgd}} &  $\min\{\frac{d}{s^2},\frac{\sqrt{d}}{s}\} \frac{1}{N}$    & $Ns(s+\sqrt{d})$   \\[1ex]
    \hline
    \makecell[l]{DME \cite{dme} }      & $\min\{\frac{1}{Ns}, \frac{\log d}{N (s-1)^2} \}$  & $Nsd$
      \\[1ex]
    \hline
    \makecell[l]{ vqSGD $Q_{C_{cp}}$}   & $\frac{d}{Ns}$ & $Ns\log d$ \\[1ex]
    \hline
     vqSGD Gaussian  & $\frac{d}{Nsc}$ & $Ns c$ \\[1ex]
    \hline
    \end{tabular}
\caption{Comparison of non private quantization schemes. \\($N$: number of worker nodes, $s$: tuning parameter $\ge 1$, $c > \log d$)) }
\label{tab:privateeq}
\end{table*}

\begin{table*}[ht]
    \centering
        \begin{tabular}{|l|c|c|l|}
    \hline 
    Method & Error & \makecell[l]{ Communication (bits) \\ (in each iteration)} & DP $(\epsilon)$\\
    \hline
      \makecell[l]{cpSGD  \cite{cpsgd}} &  $O_{\delta}\left(\frac{d}{N}\right)\footnotemark$ &  $O_{\delta}(d) $     &     \makecell[l]{$ \delta > 0$,\\ $\epsilon > f(\delta)$}    \\[1ex]
    \hline
  vqSGD $Q_{C_{cp}}$  & $O\left(\frac{d}{N}\right)$  &  $O\left(\log d\right)$     & $\eps > O(\log d)$      \\[1ex]
    \hline
    vqSGD $Q_{C_S}$  & $O\left(\frac{d^2}{N}\right)$  &   $O\left(\log d\right)$     &  $\eps > \log 7$  \\[1ex]
    \hline 
    vqSGD $Q_{C_H}$  &  $O\left(\frac{d^2}{N}\right)$  &    $O\left(\log d\right)$     &  $\eps > \log (2.5)$   \\[1ex]
    \hline 
      \makecell[l]{vqSGD $Q_{C_{cp}}$ + \RR}  &  $O\left(\frac{d^2}{N}\right)$  &    $O\left(\log d\right)$    &  $\epsilon > 0$    \\[1ex]
    \hline 
\end{tabular}
\caption{Comparison of private quantization schemes.}
 \label{tab:nonprivateq} 
\end{table*}
\footnotetext{$O_{\delta}$ hides terms involving $\delta$.}

\section{Related Work}\label{sec:rel}

 The foundations of gradient quantization was laid by \cite{seide20141} and  \cite{strom2015scalable} with schemes that require the compute nodes to send exactly  $1$-bit per coordinate of the gradient in each iteration of SGD. They also suggested using local error accumulation to correct the global gradient in every iteration. 
While these novel techniques worked well in practice, there were no theoretical guarantees provided for convergence of the scheme. 
These seminal works fueled multiple research directions. 

{\bf Quantization \& Sparsification:} \cite{qsgd,atomo, terngrad} propose stochastic quantization techniques to represent each coordinate of  the gradient using small number of bits. The proposed schemes always return an unbiased estimator of the local gradient and require $c = \Omega(\sqrt{d})$ bits of communication in each iteration of SGD (per compute node) to compute the global gradient with variance bounded by a multiplicative factor of $O(d/c)$.  
The quantization techniques for distributed SGD, can be used in the more general setting of communication efficient distributed mean estimation problem, which was the focus of \cite{dme}. The quantization schemes proposed in \cite{dme} require $O(d)$ bits of communication per compute node to estimate the global mean with a constant (independent of $d$) squared error (variance). 
Even though the tradeoff between communication and accuracy achieved by the above mentioned schemes are near optimal \cite{mjordan}, they were unable to break the $\sqrt{d}$ barrier of communication cost (in each iteration per compute node). Moreover, the schemes proposed in \cite{qsgd, dme} are variable length codes that achieve low communication in expectation. The worst case communication cost could be higher. 
In a parallel work \cite{mayekar2019ratq}, the authors propose an efficient fixed-length quantization scheme that achieves near-optimal convergence with $T$-rounds of SGD. However, the goal of their work is different from ours, and the methodologies are different as well.

In this work, we propose (fixed length) quantization schemes that require $o(d)$ (as low as $\log d$) bits of communication and are almost optimal as well. In fact for any $c$-bits of communication, the quantization scheme with Gaussian points achieves a variance of $O(d/c)$ that meets the lower bounds for any unbiased quantization scheme (also shown in the current work). 

Gradient sparsification techniques with provable convergence (under standard assumptions) were studied in \cite{acharya2019distributed, alistarh2018convergence,ivkin2019communication,stich2018sparsified}. 
The main idea in these techniques is to communicate only the top-$k$ \emph{components} of the $d$-dimensional local gradients that can be accumulated globally to obtain a good estimate of the true gradient. 
Unlike the quantization schemes described above, gradient sparsification techniques can achieve $O(\log d)$ bits of communication, but are not usually unbiased estimates of the true gradients. 
\cite{shalev2010trading} suggest randomized sparsification schemes that are unbiased, but 
are not known to provide any theoretical convergence guarantees in very low sparsity regimes. 


See Table~\ref{tab:nonprivateq} for a comparison of our results with the state of the art quantization schemes.

{\bf Error Feedback:} Many works focused on providing techniques to reduce the error incurred due to quantization \cite{diana19, karimireddy2019error} using locally accumulated errors. 
In this work, we focus primarily on gradient quantization techniques, and note that the variance reduction techniques of \cite{diana19} can be used on top of the proposed quantization schemes.

\section{Preliminaries}\label{sec:bkgrnd}
Let $[n]$ denote the set $\{1,2,\ldots ,n\}$ and let $\p{1}_d$, $\p{0}_d$ denote the all $1$'s vector and all $0$'s vector in $\R^d$ respectively.
For any $\p{x} , \p{y} \in \R^d$, we denote the Euclidean ($\ell_2$) distance between them as $\| \p{x} - \p{y} \|_2$. For any vector $\p{x} \in \R^d$,  $x_i$ denotes its $i$-th coordinate.
For any $\p{c} \in \R^d$, and $r >0$, let $B_d(\p{c}, r)$ denote a $d$-dimensional $\ell_2$ ball of radius $r$ centered at $\p{c}$. Also, let $S^{d-1}$ denote the unit sphere about $\p{0}_d$. 
Let $\p{e_i} \in \R^d$ denote the $i$-th standard basis vector which has $1$ in the $i$-th position and $0$ everywhere else. Also, for any prime power $q$, let $\mathbb{F}_q$ denote a finite field with $q$ elements.

For a discrete set of points $C \subset \R^d$, let $\conv(C)$ denote the convex hull of points in $C$, \ie
$\conv(C) := \left\{ \sum_{\p{c} \in C} a_c \p{c} \mid a_c \ge 0, \sum_{\p{c} \in C} a_c = 1 \right\}$.

Suppose $\p{w} \in \R^d$ be the parameters of a function to be learned  (such as weights of a neural network). In each step of the SGD algorithm,  the parameters are updated as $\p{w} \leftarrow \p{w} - \eta \p{\hat{g}}$, where $\eta$ is a possibly time-varying learning rate and $\p{\hat{g}}$ is a stochastic unbiased estimate of $\p{g}$, the true gradient of some loss function with respect to $\p{w}$. The assumption of unbiasedness is crucial here, that implies $\E \p{\hat{g}} = \p{g}$. 

The goal of any gradient quantization scheme is to reduce cost of communicating the gradient, i.e., to act as an first-order oracle, while not compromising too much on the \emph{quality} of the gradient estimate. 
The quality of the gradient estimate is measured in terms the convergence guarantees it provides.
In this work, we will develop a scheme that is an {\em almost surely bounded oracle} for gradients, i.e., $ \|\p{\hat{g}}\|^2_2 \le B$ with probability 1, for some $B>0$.  The convergence rate of the SGD algorithm for any convex function $f$ depends on the upper bound of the norm of the unbiased estimate, i.e., $B$, cf. any standard textbook such as  \cite{shalev2014understanding}.


Although we provide an almost surely bounded oracle as our quantization scheme, previous quantization schemes, such as \cite{qsgd}, provides a {\em mean square bounded oracle}, i.e., an unbiased estimate $\p{\hat{g}}$ of $\p{g}$ such that $\E\|\p{\hat{g}}\|_2^2 \le B$ for some $B>0$. It is known that, even with a mean square bounded oracle, SGD algorithm for a convex function converges with dependence on the upper bound $B$ (see \cite{bubeck2017convex}). As discussed in \cite{qsgd}, one can also consider the variance of $\p{\hat{g}}$ without any palpable difference in theory or practice. Therefore, below we consider the variance of the estimate $\p{\hat{g}}$ as the main measure of error.

In distributed setting with $N$ worker nodes, let $\p{g_i}$ and $\p{\hat{g_i}}$ are the local true gradient and its unbiased estimate  computed at the $i$th compute node for some $i \in \{1, \dots, N\}$. For $\p{g} =\frac{1}{N}\sum_i \p{g_i}$, the variance  of the estimate $\p{\hat{g}} = \frac{1}{N}\sum_i \p{\hat{g_i}}$ is  defined as 
$$
\text{Var}(\p{\hat{g}}) := \E \left[ \| \frac{1}{N}\sum_{i=1}^N\p{g_i} -\frac{1}{N}\sum_{i=1}^N\p{\hat{g_i}} \|_2^2 \right] 
= \frac{1}{N^2} \sum_{i=1}^N\E\left[ \| \p{g_i} - \p{\hat{g_i}} \|_2^2  \right].
$$
In this work, our goal is to design quantization schemes to efficiently compute unbiased  estimate  $\p{\hat{g}_i}$ of $\p{g}_i$ such that $\text{Var}(\p{\hat{g}})$ is minimized. 

For the privacy preserving gradient quantization schemes, we consider the standard notion of $(\eps, \delta)$-differential privacy (DP) as defined in \cite{privacybook}. 
Consider data-sets from a domain $\calx$.  Two data-sets $U, V \in \calx$, are \emph{neighboring} if they differ in at most one data point. 

\begin{definition}\label{def:privacy}
A randomized algorithm $\calm$ with domain $\calx$ is $(\eps, \delta)$-differentially private (DP) if for all $S \subset \text{Range}(\calm)$ and for all neighboring data sets $U, V \in \calx$, 
\[
\Pr[\calm(U) \in S ] \le e^{\eps} \Pr[ \calm(V) \in S] + \delta,
\]
where, the probability is over the randomness in $\calm$. If $\delta=0$, we say that $\calm$ is $\eps$-DP.
\end{definition}

We will need the notion of an $\varepsilon$-nets subsequently.
\begin{definition}[$\varepsilon$-net]\label{def:epsnet}
A set of points $N(\varepsilon) \subset {\cal S}^{d-1}$ is an $\varepsilon$-net for the unit sphere ${\cal S}^{d-1}$ if for any point $\p{x} \in {\cal S}^{d-1}$ there exists a net point $\p{u} \in N(\varepsilon)$ such that $\|\p{x}-\p{u}\|_2 \le \varepsilon$.
\end{definition}
There exist various constructions for $\varepsilon$-net over the unit sphere in $\R^d$ of size at most $\left( 1+2/\varepsilon\right)^d$ \cite{cohen1997covering}. 

\begin{definition}[Hadamard Matrix]\label{def:hadamard}
A Hadamard matrix $H_n$ of order $n$ is a $n \times n$ square matrix with entries from $\pm 1$ whose rows are mutually orthogonal. Therefore, it satisfies $HH^T = nI_n$, where $I_n$ is the $n \times n$ identity matrix. 
\end{definition}
Sylvester's construction \cite{georgiou2003hadamard} provides a recursive technique to construct Hadamard matrices for orders that are powers of $2$ which can be defined as follows. Let $H_1 = [1]$ be the Hadamard matrix of order $2^0$, and let $H_p$ denote the Hadamard matrix of order $2^p$, then a Hadamard matrix of order $2^{p+1}$ can be constructed as 
\[
H_{p+1} = \begin{bmatrix}
H_p & H_p \\
H_p & -H_p
\end{bmatrix}
\]

\section{Quantization Scheme}\label{sec:quant}
We first present our quantization scheme in full generality. Individual quantization schemes with different tradeoffs are then obtained as specific instances of this general scheme. 


Let $C = \{\p{c_1}, \dots, \p{c}_{m}\} \subset \R^d$ be a discrete set of points such that its convex hull, $\conv(C)$ satisfies
\begin{equation}\label{cond:hull}
B_d(\p{0}_d,1) \subset \conv(C)
\subseteq B_d(\p{0}_d,R),  R > 1. 
\end{equation}
Let $\p{v} \in B_d(\p{0}_d,1)$. Since $B_d(\p{0}_d,1) \subseteq \conv(C)$, we can write $\p{v}$ as a convex linear combination of points in $C$. Let 
$\p{v} = \sum_{i=1}^{m} a_i \p{c_i}, 
\mbox{ where } 
a_i \ge 0, \sum_{i=1}^{m} a_i = 1.$
We can view the coefficients of the convex combination 
$(a_1, \ldots, a_{m})$ 
 as a probability distribution 
over points in $C$. 
Define the quantization of  $\p{v}$ with respect to the set of points $C$ as follows:
\[
Q_C(\p{v}) := \p{c_i} \mbox{ with probability } a_i
\]
It follows from the definition of the quantization that $Q_C(\p{v})$ is an unbiased estimator of $\p{v}$.
\begin{lemma}\label{lem:unbiased}
$\E[Q_C(\p{v})] = \p{v}$.
\end{lemma}
\begin{proof}
$\E[Q_C(\p{v})]  = \sum_{i=1}^{|C|} a_i \cdot \p{c_i} = \p{\tilde{v}} = \p{v}.$
\end{proof}

We assume that $C$ is fixed in advance and is known to the compute nodes and the parameter server. 

\begin{remark}
Communicating the quantization of any vector $\p{v}$, amounts to sending a floating point number $\|\p{v}\|_2$, and the index of point $Q_C(\p{v})$ which requires $\log |C|$ bits. 
For many loss functions, such as Lipschitz functions, the bound on the norm of the gradients is known to both the compute nodes and the parameter server. In such settings we can avoid sending $\|v\|_2$ and the cost of communicating the gradients is then exactly $\log |C|$ bits.
\end{remark}

Any point set $C$ that satisfies Condition~\eqref{cond:hull} gives the following bound on the variance of the quantizer. 
\begin{lemma}\label{lem:hull}
Let $C \subset \R^d$ be a point set satisfying Condition~\eqref{cond:hull}. For any $\p{v} \in B_d(\p{0}_d,1)$, let $\p{\hat{v}} := Q_C(\p{v})$.  Then, $\|\p{\hat{v}}\|_2^2 \le R^2$ almost surely, and
$\E \left[ \| \p{v} - \p{\hat{v}} \|_2^2 \right] \le  R^2$.
\end{lemma}
\begin{proof}
From the definition of the quantization function,
\begin{align*}
\E[ \|\p{v} - Q_C(\p{v}) \|_2^2] &=  \E[\|Q_C(\p{v})\|^2]  - \|\p{v}\|^2 \le  R^2.
\end{align*}
This is true as $C$ satisfies Condition~\eqref{cond:hull} and therefore, each point $\p{c_i} \in C$ has a bounded norm, $\|\p{c_i}\| \le R$.  
\end{proof}

\begin{remark}
If, for any vector $\p{v}$, we send the floating point number $\|\p{v}\|_2$ separately, instead of there being a known upper bound on gradient, we can just assume without loss of generality that $\p{v} \in S^{d-1}$. 
In this case, the subsequent bounds on variance $\E \left[ \| \p{v} - \p{\hat{v}} \|_2^2 \right]  = \E \left[ \| \p{\hat{v}} \|_2^2 \right]  - \|\p{v} \|_2^2$ can be replaced by $R^2-1$.
\end{remark}

From the above mentioned properties, we get a family of quantization schemes depending on the choice of point set $C$ that satisfy Condition~\eqref{cond:hull}. For any choice of quantization scheme from this family, we get the following bound regarding the convergence of the distributed SGD.

\begin{theorem}\label{thm:hull}
Let $C \subset \R^d$ be a point set satisfying Condition~\eqref{cond:hull}. 
Let $\p{g_i} \in \R^d$ be the local gradient computed at the $i$-th node, 
Define $\p{\hat{g}} := \frac1N \sum_{i=1}^N \p{\hat{g}_i}$, where  $\p{\hat{g}_i} := \|\p{g_i}\| \cdot Q_C(\p{g_i} / \|\p{g_i}\|)$. Then, 
\begin{align*}
\E[\p{\hat{g}} ] = \p{g} \quad \mbox{ and} \quad \E \left[ \| \p{g} - \p{\hat{g}} \|_2^2 \right]  \le (R/N)^2 \sum_i \| \p{g_i}\|^2.
\end{align*}
\end{theorem}
\begin{proof}
Since $\p{\hat{g}}$ is the average of $N$ unbiased estimators, the fact that $\E[\p{\hat{g}}] = \p{g}$ follows from Lemma~\ref{lem:unbiased}.
For the variance computation, note that 
\begin{align*}
\E[\|\p{g} - \p{\hat{g}}\|_2^2] &=\frac{1}{N^2} \left( \sum_{i=1}^N \E[\| \p{g_i} - \p{\hat{g_i}}\|_2^2 ]  \right) \qquad (\mbox{ since } \p{\hat{g_i}} 
\mbox{ is an unbiased estimator of $\p{g}$ })\\
&\le \frac{R^2}{N^2} \sum_{i=1}^N  \| g_i\|^2  \qquad \mbox{(from Lemma~\ref{lem:hull})}.
\end{align*}
\end{proof}

\begin{remark} 
Computing the quantization $Q_C(.)$ amounts to solving a system of $d+1$ linear equations in $\R^{|C|}$. 
For general point sets $C$, this takes about $O(|C|^3)$ time (since $|C| \ge d+1$). However, we show that for certain structured point sets, the quantization $Q_C(.)$ can be computed in linear time.  
\end{remark}

From Theorem~\ref{thm:hull} we observe that the communication cost of the quantization scheme depends on the cardinality of $C$ while the convergence is dictated by the circumradius $R$ of the convex hull of $C$. In the Section~\ref{sec:constructions}, we present several  constructions of point sets which provide varying tradeoffs between communication and variance of the quantizer. 

\vspace{-5pt}\paragraph{Reducing Variance: }\label{sec:repetition}
In this section, we propose a simple repetition technique to reduce the variance of the quantization scheme. 
 For any $s > 1$, let $Q_C(s, \p{v}) := \frac 1s \sum_{i=1}^s Q_C^{(i)}(\p{v})$ be the average over $s$ independent applications of the quantization $Q_C(\p{v})$. Note that even though $Q_C(s, \p{v})$  is not a point in $C$, we can communicate $Q_C(s, \p{v})$ using an equivalent representation as a tuple of $s$ independent applications of $Q_C(\p{v})$ that requires $s \log |C|$ bits. 
Using this repetition technique 
we see that the variance reduces by factor of $s$ while the communication increases by the exact same factor.  
\begin{proposition}\label{prop:var}
Let $C \subset \R^d$ be a point set satisfying Condition~\eqref{cond:hull}. For any $\p{v} \in B_d(\p{0}_d,1)$,  and any $s \ge 1$, let $\p{\hat{v}} := Q_C(s, \p{v})$.  Then, 
$\E \left[ \| \p{v} - \p{\hat{v}} \|_2^2 \right] \le R^2/s$.
\end{proposition}
\begin{proof}
The proof follows simply by linearity of expectations. 
\begin{align*}
\E \left[ \| \p{v} - \p{\hat{v}} \|_2^2 \right] &= \E \left[ \| \frac1s \sum_{i=1}^s \left(\p{v} - Q_C(\p{v}) \right) \|^2 \right] \le \frac{1}{s} \cdot R^2  \qquad  \mbox{(from Lemma~\ref{lem:hull})}. 
\end{align*}
\end{proof}


\section{Constructions of Point Sets and Lower Bound}\label{sec:constructions}
In this section, we propose constructions of  point sets that satisfy Condition~\eqref{cond:hull} and provide varying tradeoffs between communication and variance of the quantization scheme. 
But first, we start with a lower bound that shows that one must communicate $\Omega(\frac{d}{R^2})$ bits to achieve an error of $O(R^2)$ in the estimate of the gradient as per Condition~\eqref{cond:hull}.

\begin{theorem}\label{thm:lb}
Let $C  \subseteq \mathbb{R}^d$ be a discrete set of points that satisfy Condition~\eqref{cond:hull}. Then 
$$
|C| \ge \exp(\alpha d /R^2)
$$
for some absolute constant $\alpha>0$.
\end{theorem}

To prove the lower bound, we show a strong characterization of the point sets that satisfy Condition~\eqref{cond:hull}, and later use this characterization to construct point sets with optimal tradeoffs. 
\begin{theorem}\label{thm:char}
Let $C = \{c_1, \ldots, c_m\} \subseteq \mathbb{R}^d$ be a discrete set of points. The unit ball $B_d(\p{0}_d,1) \subseteq \textsc{conv}(C)$ if and only if for all points $\p{x} \in S^{d-1}$, there exists a point $\p{c} \in C$ such that $\langle \p{x},\p{c} \rangle \ge 1$.
\end{theorem}
\begin{proof}
Assume that for some $\p{x} \in S^{d-1}$, $\langle \p{x}, \p{c} \rangle < 1$ for all $\p{c} \in C$. Which implies that all points of $C$, and therefore the $\conv(C)$, are separated from $\p{x}$ by the hyperplane $ H_w := \{ \p{w} \in \R^d  | \langle \p{x}, \p{w} \rangle = 1\} $. Therefore $\p{x} \notin \conv(C)$. 

To prove the other side, assume $B_d(\p{0}_d,1) \not\subset \conv(C)$. Let $H_w := \{ \p{z} \in \R^d | \langle \p{w}, \p{z} \rangle =1 \}$ be the separating hyperplane that partitions $B_d(\p{0}_d,1)$ such that $\conv(C)$ lies on one side of the hyperplane. Without loss of generality, we assume $ \conv(C) \subset H_w^- := \{ \p{z} \in \R^d |\langle \p{w}, \p{z} \rangle < 1\}$. 
Since $H_w$ partitions the unit ball, the distance of $H_w$ from the origin is $1/\|\p{w}\| \le 1$. 

Now consider the point $\p{x} := \p{w} / \|\p{w}\| \in S^{d-1}$. For this point, $\langle \p{x}, \p{c} \rangle = \frac{1}{\| \p{w}\|} \langle \p{w}, \p{c} \rangle < 1$ for all $\p{c} \in C$.  
\end{proof}

\begin{proofof}{Theorem~\ref{thm:lb}}

The proof of this theorem will use a packing argument for $S^{d-1}$. Let $\p{c} \in C$. We will estimate the \emph{cardinality} of the set $P(\p{c}) :=\{\p{x} \in S^{d-1}: \langle \p{x},\p{c} \rangle \ge 1\} $ under the uniform measure over $S^{d-1}$. Using Theorem \ref{thm:char}, size of $C$ must be at least $$\frac{{\rm area}(S^{d-1})}{\max_{\p{c}\in C}{\rm area}(P(\p{c}))}$$ for it to satisfy Condition~\eqref{cond:hull}.

Note that, $P(\p{c})$ is a hyperspherical cap with angle $\phi$ such that $\cos\phi \ge \frac{1}{\|\p{c}\|} \ge \frac{1}{R}$, since $C$ satisfy Condition~\eqref{cond:hull}. The area of a cap can be computed using the incomplete beta functions, however a probabilistic argument below will serve to lower bound this.

If we uniformly at random choose a vector $\p{z}$ from $S^{d-1}$, then the probability $p$ that it is within an angular distance $\phi$ of a fixed unit vector, $\p{u}$, will exactly be the the ratio of the areas of the hyperspherical cap and the sphere. Again this probability is known to follow a shifted Beta distribution, but we can estimate it from above using concentration bounds. 

Since the area of the hyperspherical cap is invariant to its center, we can take $\p{u}$ to be the first standard basis vector. It is known that if $\p{g}= (g_1,g_2, \dots, g_d)\in \R^d$ is a random vector with i.i.d. Gaussian $\mathcal{N}(0,1)$ entries, then $\p{z}:=\p{g}/\|\p{g}\|$ is uniform over $S^{d-1}$. Therefore,
\begin{align*}
\Pr(\langle \p{z},\p{u} \rangle \ge 1/R) &= \Pr(g_1/\|\p{g}\| \ge 1/R)\\
&\le \Pr(g_1 \ge \|\p{g}\|/R \mid \|\p{g}\| \ge \sqrt{d}/4) + \Pr( \|\p{g}\| < \sqrt{d}/4)\\
&\le \Pr(g_1 \ge \sqrt{d}/4R) + \Pr( \|\p{g}\| < \sqrt{d}/4).
\end{align*}
Now since $g_1$ is  $\mathcal{N}(0,1)$, $\Pr(g_1 \ge \frac{\sqrt{d}}{4R}) \le \exp(-\frac{d}{32R^2}),$ from Chernoff bound. On the other hand $\|\p{g}\|^2$ is a $\chi^2$ distribution of $d$ degrees of freedom. Since that is subexponential, we have $\Pr( \|\p{g}\| < \sqrt{d}/4) \le \Pr( \|\p{g}\|^2< d/16) \le \exp(-\frac{225d}{2048}) \le \exp(-\frac{d}{32R^2})$ for any $R \ge 1$.

This implies,
$
|C| \ge (2\exp(-d/ (32R^2)))^{-1}.
$
\end{proofof}


\subsection{Gaussian point set}
We provide a  randomized construction of point set using the characterization defined above, that is optimal in terms of communication.
\begin{theorem}\label{thm:random-const}
Let $R \in [5, 6\sqrt{d}]$. There exists a set $C$ of $\exp(O(d/R^2 + \log d))$ points of $\ell_2$ norm at most $R$ each, that satisfy Condition~\eqref{cond:hull}.
\end{theorem}

To prove Theorem~\ref{thm:random-const}, we first show the following lemma that allows us to union bound over the discrete set of points in an $\varepsilon$-net of a unit sphere.  Consider an $\varepsilon$-net for the unit sphere $N(\varepsilon)$ for any $\varepsilon < 1/R$. We know that such a set exists with $|N(\varepsilon)| \le \left( 1+\frac{2}{\varepsilon}\right)^d \le \left(\frac{3}{\varepsilon}\right)^d$ \cite{cohen1997covering}.  
 
\begin{lemma}\label{lem:witnessnet}
Let $C$ be a set of points in $\R^d$ such that $\|c\|^2 \le R^2$ for all $c \in C$. 
If for each $y \in N(\varepsilon)$,  $y^T c \ge 2$ for some $c \in C$, then it holds that for all $x \in S^{d-1}$, there is a $c' \in C$ such that $x^T c' \ge 1$. 
\end{lemma}

\begin{proof}
Let $\p{y} \in N(\varepsilon)$ be a net-point, and $\p{c} \in C$ be such that $\p{x}^T c \ge 2$. 
Note that all points $\p{x} \in S^{d-1}$ in the $\varepsilon$-neighborhood of $\p{y}$ can be written as $\p{x} = \p{y}+ \p{\tilde{y}}$, 
where $ \p{\tilde{y}} \in \R^d$ has norm at most $\epsilon$. 
Therefore, $\p{x}^T \p{c}  = \p{y}^T \p{c}  + \p{\tilde{y}}^T \p{c} \ge 2 - \|\p{\tilde{y}}\| \|\p{c}\| > 1$. 
Since $N(\varepsilon)$ covers the entire unit sphere, it follows that for all points $\p{x}$ on the unit sphere, there will be a $\p{c} \in C$ such that $\p{x}^T\p{c} \ge 1$. 
\end{proof}

\begin{proofof}{Theorem~\ref{thm:random-const}}

Let us choose the random set $C$ of $t:= e^{\frac{20d}{R^2}+2\log d}$ points in $\R^d$ in the following way: Each coordinate of any $\p{c} \in C$ is chosen independently from a zero-mean Gaussian distribution with variance $\sigma^2:=\frac{R^2}{9d}$. 

We say that a vector $\p{x} \in S^{d-1}$ is a witness for $C$ if $\p{x}^T \p{c} < 1$ for all  $\p{c} \in C$. We now show that with high probability, there is no witness for $C$ in $S^{d-1}$. Using Lemma~\ref{lem:witnessnet}, it is sufficient to show that for any $\p{x} \in N(\varepsilon)$, $\p{x}^T \p{c}\ge 2 $ for some $\p{c} \in C, \varepsilon \le 1/R$.

Let us define $E_1$ to be the event that $\|\p{c}\|^2 \le R^2$ for all $\p{c} \in C$. Since every entry of $\p{c}$ is chosen from i.i.d. Gaussian, the norm $\|\p{c}\|^2$ is distributed according to $\chi^2$-distribution
with variance $2d\sigma^4$. Since $\chi^2$-distribution is subexponential ~\cite{wainwright2019high}[ Eq.~2.18], for any $\p{c} \in C$, we have, for any $l \ge 1$,
$
\Pr(\|\p{c}\|^2 > d\sigma^2(l+1)) \le e^{-dl/8}.
$
This implies, 
$$
\Pr(\|\p{c}\|^2 > R^2) \le e^{-\frac18(R^2/\sigma^2 -d)} \le e^{-d},
$$
substituting the value of $\sigma^2$. 
 Then by union bound over all $c \in C$.  
\begin{align}\label{eq:E1bar}
\Pr[ \bar{E}_1] \le t e^{-d} = e^{-d+20d/R^2 +2\log d} \le e^{-\Omega(d)},
\end{align}
for $R\ge 5$.



Let $E_2$ denote the event that for each $\p{y} \in N(\varepsilon)$, there exists $\p{c} \in C$ such that $\p{y}^T \p{c} \ge 2 $.
For any fixed $\p{y} \in N(\varepsilon)$, and $\p{c} \in C$, define $p_{\p{y}, \p{c}}$ to be the probability that  $\p{y}^T \p{c} \ge 2$. 
Note that since $c$ has i.i.d. Gaussian entries, then for any fixed $\p{y} \in N(\varepsilon)$, the inner product $\p{y}^T\p{c}$ is distributed according to ${\cal N}(0, \sigma^2)$. 
Using standard bounds for Gaussian distributions~\cite{borjesson1979simple},
\begin{align*}
p_{\p{y}, \p{c}} &:= \Pr[\p{y}^T\p{c} \ge 2] \ge \frac{2\sigma}{(\sigma^2+4)\sqrt{2\pi}} e^{-\frac{2}{\sigma^2}} \\
&\ge \frac{1}{\sqrt{2\pi}} \min(\sigma^{-1}, \sigma/4)e^{-\frac{2}{\sigma^2}} \ge  \frac{R}{12\sqrt{2\pi d}} e^{-\frac{2}{\sigma^2}},
\end{align*} 
for any $R \le 6 \sqrt{d}$.

Since each $\p{c}$ is chosen independently, the probability that $\p{y}^T\p{c} < 2$ for all $\p{c} \in C$ is $(1 - p_{\p{y}, \p{c}})^{t} \le e^{- t \cdot p_{\p{y}, \p{c}}}$. 
Now  for $\varepsilon = 1/R$, by union bound, since $|N(\varepsilon)| \le \left(\frac{3}{\varepsilon}\right)^d$, 
\begin{align*}
\Pr[\bar{E_2}] &=\Pr[ \exists~\p{y} \in N(\varepsilon) \text{ s.t. } \p{y}^T \p{c} < 2 ~\forall~\p{c}\in C ] \\
&\le e^{- t \cdot p_{y,c} + d \log 3R}  \\
&=  e^{- t \cdot e^{-\frac{18 d}{R^2} - \log( \frac{12\sqrt{2\pi d}}{R}) } + d \log 3R } \\
& = e^{- e^{\frac{2 d}{R^2} +2\log d - \log( \frac{12\sqrt{2\pi d}}{R}) } + d \log 3R }\\
& = e^{- d^2e^{\frac{2 d}{R^2} - \log( \frac{12\sqrt{2\pi d}}{R}) } + d \log 3R }\\
&\le e^{-\Omega(d)}.  
\end{align*}
Therefore, $\Pr[\bar{E_1} \cup \bar{E_2}]  \le e^{-\Omega(d)}$. Then using Lemma~\ref{lem:witnessnet} and Theorem~\ref{thm:char}, we obtain  the statement of the theorem.



\end{proofof}

The above stated theorem provides a randomized algorithm to generate a point set of size about $\exp(\Theta(d/R^2))$ such that the quantization scheme defined in Section~\ref{sec:quant} instantiated with this point set achieves a variance of $O(R^2)$ while communicating $\tilde{O}(d/R^2)$ bits, hence meeting the lower bound of Theorem~\ref{thm:lb}. 
In particular, there exists a quantization scheme that achieves $O(1)$ variance with $\tilde{O}(d)$ bits of communication (see supplementary material for a deterministic construction). Also, at the cost of communicating only $O(\log d)$ bits, our quantization scheme can achieve a variance of $O(d/\log d)$. The deterministic constructions we provide (in Sec.~\ref{sec:derandomRM}, and also $Q_{cp}$ in the supplement), meet this bound up to a factor of $\log d$. 

\subsection{Derandomizing with Reed Muller Codes}\label{sec:derandomRM}
In this section, we propose a deterministic construction of point set based on first order Reed-Muller codes that satisfy Condition~\ref{cond:hull}. 
We assume $d$ to be a power of $2$, \ie $d = 2^p$ for some $p\ge1$. 

Our quantization scheme is based on the first order Reed-Muller codes, $\textsc{RM}(1, p)$~(\cite{macwilliams1977theory}). Each codeword of $\textsc{RM}(1, p)$ is given as the evaluations of a degree $1$, $p$-variate polynomial over all points in $\F_2^p$. Mapping these codewords to reals using the coordinate-wise map $\phi: \F_2 \rightarrow \R$ defined as $\phi(b) = (-1)^b$ will give us a set of $ 2d$ points in $\{\pm1\}^d$. Let $C_{RM}$ denote this set of mapped codewords. 

We show that the set of points in $C_{RM}$ satisfy the characterization of Theorem~\ref{thm:char}, and therefore will give us a quantization scheme with $\log 2d$ communication and the following guarantees: 

\begin{proposition}\label{prop:RMcode}
For any $\p{v} \in B_d(\p{0}_d,1)$, let $\p{\hat{v}} := Q_{C_{RM}}(\p{v})$.  Then, 
$\E[\p{\hat{v}}] = \p{v}$  and,  $\E \left[ \| \p{v} - \p{\hat{v}} \|_2^2 \right] = O(d)$. 
\end{proposition}

\begin{proof}
We prove this theorem by showing that the point set $C_{RM}$ satisfies the characterization of Theorem~\ref{thm:char}. Since all points in $C_{RM}$ have squared norm exactly $d$, from Lemma~\ref{lem:unbiased} and Lemma~\ref{lem:hull}, the proof follows. 

First note that the matrix with the points in $C_{RM}$ as its rows, has the following structure: 
\[
H := \begin{bmatrix}
H_p \\ -H_p
\end{bmatrix}
\]
where, $H_p$ is the $2^p \times 2^p$ Hadamard matrix. 

For any fixed $\p{x} \in S^{d-1}$, consider the sum $S(\p{x}):=\sum_{\p{c} \in C_{RM}} ( \p{x}^T \p{c})^2$. We first show that $S(\p{x}) \ge 2(d+1)$. 
\begin{align*}
S(\p{x}) &= \sum_{\p{c} \in C_{RM}} ( \p{x}^T \p{c})^2 
= 2 \sum_{\p{h_i} \in H_p} ( \p{x}^T \p{h_i})^2 \\
&=  2 \| H_p \p{x} \|^2 
= 2 (\p{x}^T H_p^T)(H_p \p{x}) \\
& \overset{(i)}{=} 2d \cdot \|x\|^2 = 2 d.
\end{align*}
$(i)$ follows from the fact that the columns of the Hadamard matrix are mutually orthogonal and therefore, $H_p^T H_p = d\cdot I_d$, where, $I_d$ denotes the $d \times d$ identity matrix. 

By an averaging argument, it then follows that there exists at least one $\p{c} \in  C_{RM}$ such that $\lvert \p{x}^T \p{c} \rvert \ge 1$. Since for every $ \p{c} \in C_{RM}$, there exists $-\p{c} \in C_{RM}$, we get that $x^T\p{c}  \ge 1$ for some $\p{c} \in C_{RM}$.
\end{proof}

\begin{remark}\label{rem:RMcodes}
Instead of first order Reed-Muller codes, 
we can use any binary linear code $\textsc{C} \subseteq \F_2^d$ to construct the point set as follows. Map all the codewords from $\F_2^d$  to $\R^d$ using $\phi$ described above. The point set containing all such mapped codewords, and their complements will give a quantization scheme with variance $O(d)$. The communication will however be $\log (2|\textsc{C}|)$, where $|\textsc{C}|$ denotes the number of codewords in $\textsc{C}$. In this regard, the first order Reed-Muller codes described above provide the best communication guarantees and the quantization is also efficiently computable.
\end{remark}
\subsection{Other Deterministic Constructions}\label{sec:det-const}
We now present several explicit constructions of point sets that give quantization schemes with varying tradeoffs. On one end of the spectrum, the cross-polytope scheme requires only $O(\log d)$ bits to communicate an unbiased estimate of a vector in $\R^d$ with variance $O(d)$.  While on the other end, the $\varepsilon$-net based scheme achieves a constant variance at the cost of $O(d)$ bits of communication. 


\subsubsection{Cross Polytope Scheme}\label{sec:cp}
Consider the following point set of $2d$ points in $\R^d$:
\[
C_{cp} := \{ \pm \sqrt{d}~\p{e_i}  \mid  i \in [d] \}, 
\]
The convex hull $\conv(C_{cp})$ is a scaled cross polytope that satisfies Condition~\eqref{cond:hull} with $R=\sqrt{d}$.  
Let $Q_{C_{cp}}$ be the instantiation of the quantization scheme described in Section \ref{sec:quant} with the point set $C_{cp}$. 

To compute the convex combination of any point $\p{v} \in \conv(C_{cp})$, we need a non-negative solution to the following system of equations
\begin{align}\label{eq:lp_cp}
\begin{bmatrix} \sqrt{d}I_d & -\sqrt{d}I_d \end{bmatrix} \begin{bmatrix}a_1\\ \vdots \\ a_{2d} \end{bmatrix} = \begin{bmatrix} v_1\\ \vdots \\ v_{d}  \end{bmatrix} \quad \mbox{such that } \sum_{i=1}^{2d} a_i = 1,
\end{align}
where, $I_d$ is the $d \times d$ identity matrix. 
Equation~\ref{eq:lp_cp} leads to the following closed form solution that can be computed in $O(d)$ time:  
\begin{align} \label{eq:solution_cp}
a_i =
\left\{
	\begin{array}{ll}
		\frac{v_i}{\sqrt{d}} + \frac{\gamma}{2d}  & \mbox{if } v_i > 0  \mbox{ and } i \le d\\[3pt]
		- \frac{v_i}{\sqrt{d}} + \frac{\gamma}{2d}  & \mbox{if } v_i \le 0  \mbox{ and } i > d\\[3pt]
		\frac{\gamma}{2d}  & \mbox{otherwise } \\[3pt]
	\end{array}
\right.
\end{align}
 where, $\gamma := 1 -\frac{\|\p{v}\|_1}{\sqrt{d}}$, is a non-negative quantity for every $\p{v} \in B_d(\p{0}_d,1)$.
 
The bound on the variance of the quantizer follows directly from Lemma~\ref{lem:hull}. 
\begin{proposition}\label{prop:crosspolytope}
For any $\p{v} \in B_d(\p{0}_d,1)$, let $\hat{\p{v}} := Q_{C_{cp}}(\p{v})$. Then, 
$\E[\p{\hat{v}}]=  \p{v}$  and $\E\left[ \| \p{v} - \p{\hat{v}} \|_2^2 \right] = O(d).$
\end{proposition}

\begin{proof}
The proof of Proposition~\ref{prop:crosspolytope} follows directly from Lemma~\ref{lem:hull} provided  the point set $C_{cp}$ satisfies Condition~\eqref{cond:hull} with $R = \sqrt{d}$. We will now prove this fact. 

Since each vertex is of the form $\pm \sqrt{d} \p{e_i}$, it follows that all the vertices of $\conv(C_{cp})$, and hence the entire convex hull lies inside a ball of radius $\sqrt{d}$, \ie, $\conv(C_{cp}) \subset B_d(\p{0}_d, \sqrt{d})$.

To prove that the unit ball is contained in the convex hull $\conv(C_{cp})$, 
we pick any arbitrary point $v \in B_d(\p{0}_d,1)$ and show that it can written as a convex combination of points in $C_{cp}$. The fact follows from the solution to the system of linear equations~\eqref{eq:lp_cp}  given in Equation~\eqref{eq:solution_cp}.
Note that the solution satisfies $a_i \ge 0$ and $\sum_i a_i = 1$ for any point $\p{v} \in B_d(0,1)$. 
\end{proof}

Moreover, using the variance reduction technique described in Section~\ref{sec:quant} with $s = O(\frac{d}{\log d})$, the cross polytope based quantization scheme $Q_{C_{cp}}$ achieves a variance of $O(\log d)$ at the cost of communicating $O(d)$ bits. 

We note that the cross-polytope quantization scheme described above when used along with the variance reduction technique (by repetition), is in essence similar to Maurey sparsification~(\cite{acharya2019distributed}). 

\subsubsection{Scaled $\varepsilon$-nets}\label{subsec:epsnet}
On the other end of the spectrum, we now show the existence of points sets of exponential size that are contained in a constant radius ball. This point set allows us to obtain a gradient quantization scheme with $O(d)$ communication and $O(1)$ variance. 
We show that an appropriate constant scaling of an $\varepsilon$-net points (see Definition~\ref{def:epsnet}) satisfies Condition~\eqref{cond:hull}. 
\begin{lemma}\label{lem:scaling}
For any $0< \varepsilon < 1$, let  
$R = \frac{1}{1- \varepsilon}$. The point set $C_{\text{net}}:= \{ R \cdot \p{u} \mid \p{u} \in N(\varepsilon)\}$ satisfies Condition~\eqref{cond:hull}.
\end{lemma}
\begin{proof}
Let $K:= \conv(N(\varepsilon))$ be the convex hull of the $\varepsilon$-net points of the unit sphere. Let $B_d(\p{0}_d, r)$ be the inscribed ball in $K$ for some $r < 1$. We show that $r \ge 1-\varepsilon$.
 
Consider the face of $K$ that is tangent to $B_d(\p{0}_d,r)$ at point $\p{z}$. We will show that $\|\p{z}\|_2 \ge 1  - \varepsilon$. 
Extend the line joining $(\p{0}_d, \p{z})$ to meet ${\cal S}^{d-1}$ at point $\p{x}$. 
Since  $\p{x} \in {\cal S}^{d-1}$, we know that there exists a net point $\p{u}$ at a distance of at most $\varepsilon$ from it. Therefore, the distance of $\p{x}$ from $K$ is upper bounded by $\varepsilon$, \ie ~$\text{dist}(\p{x}, K)  = \|\p{x}-\p{z}\| \le \|\p{x}-\p{u}\| \le \varepsilon$. Therefore $\|\p{z}\|= 1-\|\p{x}-\p{z}\| \ge 1-\varepsilon$. 
 
Therefore scaling all the points of $N(\varepsilon)$ by any $R \ge \frac{1}{1-\varepsilon}$ we see that $B_d(\p{0}_d,1) \subseteq \conv(C)$.
\end{proof}

Let $Q_{\text{net}}$ be the instantiation of the quantization scheme  with point set $C_{\text{net}}$. 
From Lemma~\ref{lem:hull}, we then directly get the following guarantees for the quantization scheme obtained from scaled $\varepsilon$-nets,  $C_{\text{net}}$ for some constant $\varepsilon < 1$. 
\begin{proposition}\label{prop:epsnet}
For any $\p{v} \in B_d(\p{0}_d,1)$, let $\hat{\p{v}} := Q_{C_{\text{net}}}(\p{v})$. Then, $\E[\p{\hat{v}}]=  \p{v}$ and $\E\left[ \| \p{v} - \p{\hat{v}} \|_2^2 \right] = \frac{1}{(1- \varepsilon^)}$.
\end{proposition}
Moreover,  $Q_{\text{net}}$ requires $O(d \log \frac{1}{\varepsilon})$ bits to represent the unbiased gradient estimate. 


\section{Private Quantization}\label{sec:private}
In this section we show that under certain conditions the quantization scheme $Q_C(.)$ obtained from the point set $C$ is also $\eps$-differentially private. 
First, we see why the quantization scheme described in Section \ref{sec:quant} is not privacy preserving in general. 

Let $C$ be any point set with $|C| > d+1$. For any point $\p{x} = \sum_{i=1}^{|C|} a_i \p{c_i} \in \conv(C)$, let $\supp{\p{x}}{C} = \{ \p{c}_i  \in  C \mid a_i \neq 0 \}$ denote the points in $C$ that are in the range of $Q_C(\p{x})$.  

In order for $Q_C$ to be $\eps$-DP for any $\eps > \eps_0$, we have to show for gradients 
$\p{x}, \p{y} \in \R^d$ of any two neighboring datasets and for any $\p{z} \in \supp{\p{x}}{C} \cup \supp{\p{y}}{C}$,
\begin{align}\label{eq:priv}
 \Pr[Q_C(\p{x}) = \p{z}] \le e^{\eps_0} \cdot \Pr[Q_C(\p{y}) = \p{z}].
 \end{align}
If $|C| > d+1$, there may exist two gradients $\p{x}, \p{y} \in \conv(C)$ such that $\supp{\p{x}}{C} \neq \supp{\p{y}}{C} $. Therefore, for any $\p{z}$ in the symmetric difference of the sets $\supp{\p{x}}{C}$ and $\supp{\p{y}}{C}$, Eq.~\eqref{eq:priv} will not hold for any finite $\eps_0$. 

The discussion above establishes a sufficient condition for the quantization scheme $Q_C$ to be differentially private. Essentially, we want all points in $B_d(\p{0}_d,1)$ to have full support on all the points in $C$. This is definitely possible when $|C| = d+1$. 
Therefore if the point set satisfying Condition~\eqref{cond:hull} has size $|C| = d+1$, then the quantization scheme $Q_C$ is $\eps$-differentially private, for some $\eps > \eps(C)$.

We now present two constructions of point sets $C$ of size exactly $d+1$ satisfying Condition~\eqref{cond:hull} that give an $\eps$-differentially private quantization scheme. Both the schemes achieve a communication cost of $\log (d+1)$, but the variance is a factor $d$ larger than the non-private scheme, $Q_{C_{cp}}$. 

\paragraph{(1) Simplex Scheme: }
Consider the following set of $d+1$ points
\[
C_{S} = \{ 2 d~\p{e_i} \mid i \in [d] \} \cup \{ - 4\p{1}_d \}.
\]
The convex hull of $C_S$ satisfies Condition~\eqref{cond:hull} with $R = O(d)$ . 
Since the size of the set is exactly $d+1$, every point in the unit ball can be represented as a convex combination of all the points in $C_S$ (\ie all  coefficients of the convex combination are non zero). This fact will be used crucially to show that this scheme is also differentially private. 

The coefficients of the convex combination of any point $\p{v} \in \conv(C_S)$ can be computed from the following system of linear equations: 
\begin{align}\label{eq:lp_cs}
\begin{bmatrix} -4 \p{1}_d^T & 2 d I_d \end{bmatrix} \begin{bmatrix}a_0 \\ \vdots \\ a_{d} \end{bmatrix} = \begin{bmatrix} v_1\\ \vdots \\ v_{d}  \end{bmatrix}
\mbox{ such that } \sum_{i=0}^{d} a_i = 1.
\end{align}
Equation~\ref{eq:lp_cs} leads to the following closed form solution that can be  computed in linear-time: 
\begin{align}\label{eq:simplex_sol}
a_0 = \frac{1}{3} - \frac{(\sum_{i=1}^d v_i)}{6d} \quad a_i = \frac{v_i}{2d}+  \frac{2a_0}{d} \quad \forall i\ge 1.
\end{align}

\begin{proposition}\label{prop:simplex}
For any $\p{v} \in B_d(\p{0}_d,1)$, let $\p{\hat{v}} := Q_{C_S}(\p{v})$.  Then, 
$\E[\p{\hat{v}}] = \p{v}$ and  $\E \left[ \| \p{v} - \p{\hat{v}} \|_2^2 \right] = O(d^2)$.
Moreover, $Q_{C_S}$ is $\eps$-DP for any $\eps > \log 7$. 
\end{proposition}
\begin{proof}
First we show that the point set $C_{S}$ satisfies Condition~\eqref{cond:hull} with $R=2d$. 
The fact that $\conv(C_S) \subset B_d(\p{0}_d, 2d)$ follows trivially from the observation that each point in $C_S \in B_d(\p{0}_d, 2d)$.

To show that $B_d(\p{0}_d,1) \subset \conv(C_S)$, consider any face of the convex hull, $F_{\p{c}} := \conv(C_S \setminus \{ \p{c} \} )$, for some $\p{c} \in C_S$.  We show that $F_{\p{c}}$ is at an $\ell_2$ distance of at least $1$ from $\p{0}_d$. This in turn shows that any point outside the convex hull must be outside the unit ball as well.  

First consider the case when $\p{c} = -4\p{1}_d$. We observe that the face $F_{\p{c}}$ is contained in the hyperplane $H_{\p{c}}:= \{ \p{x} \in \R^d \mid \frac{1}{\sqrt{d}}{\bf 1}_d^T \p{x} = 2\sqrt{d} \}$, and therefore is at a distance of $O\left(\sqrt{d} \right)$ from the origin. 

Now consider the case when $\p{c} = 2d~\p{e_1}$. Let $\p{w} = \frac{1}{\sqrt{\frac{9}{16} - \frac{1}{2d} }}(-\frac 34 + \frac{1}{2d}, \frac{1}{2d}, \ldots, \frac{1}{2d})^T \in \R^d$ be a unit vector. 
 We note that $F_{\p{c}} \subset H_{\p{c}}$, where $H_{\p{c}} := \{ \p{x} \in \R^d \mid \p{w}^T \p{x} = \frac{1}{\sqrt{\frac{9}{16} - \frac{1}{2d} }}  \}$ is the hyperplane defined by the unit normal vector $\p{w}$ that is at a distance of at least $1$ from $\p{0}_d$. 
 
Since all other faces are symmetric, the proof for the case $\p{c} = 2d~\p{e_i}, i \in [d]$ follows similarly. 

\paragraph{Privacy: } We now show that the quantization scheme is $\eps$-differentially private for any $\eps > \log 7$. From the definition of $\eps$-DP, it is sufficient to show that for any $\p{x}, \p{y} \in B_d(\p{0}_d,1)$ , and every $\p{c} \in C_S$,
\[
 \frac{\Pr[Q_{C_S}(\p{x}) = \p{c}]}{\Pr[Q_{C_S}(\p{y}) = \p{c}]}  \le 7.
\]
Since $\p{x}, \p{y} \in \conv(C_S)$, we can express them as the convex combination of points in $C_S$. Let $\p{x} = \sum_{\p{c} \in C_S} a_c^{(\p{x})} \p{c}$. Similarly, let $\p{y} = \sum_{\p{c} \in C_S} a_c^{(\p{y})} \p{c}$. Then, from the construction of the quantization function $Q_{C_S}$, we know that 
\[ 
\frac{\Pr[Q_{C_S}(\p{x}) = \p{c}]}{\Pr[Q_{C_S}(\p{y}) = \p{c}]} = \frac{a_c^{(\p{x}) } }{a_c^{(\p{y})} }. 
\]
We now show that the ratio $\frac{a_c^{(\p{x}) } }{a_c^{(\p{y})} }$ is at most $7$ for \emph{any} pair $\p{x}, \p{y} \in B_d(\p{0}_d,1)$ and any $\p{c} \in C_S$. The privacy bound follows from this observation. 

First, consider the case  $\p{c} = -4\p{1}_d$. From the closed form solution for any $\p{x} \in \conv(C_S)$ described in Equation~\eqref{eq:simplex_sol}, we know that $a_c^{(\p{x})} = \frac 13 - \frac{\sum_{i=1}^d x_i}{6d}$. For any $\p{x} \in B_d(\p{0}_d,1)$, $\sum_{i=1}^d x_i \in \left[-\|\p{x}\|_1, \|\p{x}\|_1 \right] \subseteq \left[-\sqrt{d}, \sqrt{d} \right]$. Therefore, $a_c^{(\p{x})} \in \left[ \frac13-\frac{1}{6\sqrt{d}},  \frac13+\frac{1}{6\sqrt{d}}  \right]$. It then follows that for any $\p{x}, \p{y} \in B_d(\p{0}_d, 1)$ and $\p{c} = -4\p{1}_d$,
\begin{align*}
 \frac{a_c^{(\p{x}) } }{a_c^{(\p{y})} } &\le \frac{  \frac13+\frac{1}{6\sqrt{d}} }{  \frac13-\frac{1}{6\sqrt{d}}} = 1+\frac{2}{2\sqrt{d} - 1} \le 3
\end{align*}

Now we consider the case when $\p{c} = 2d~\p{e_1}$. Then from the closed from solution in Equation~\eqref{eq:simplex_sol}, we get that for any $\p{x} \in \conv(C_S)$ the coefficient $a_c^{(\p{x})} = \frac{x_1}{2d} \left(1 - \frac{2}{3d} \right) - \frac{\sum_{i=2}^d x_i }{3d^2} + \frac{2}{3d}$. Note that this quantity is maximized for $\p{x} = \p{e_1}$ and minimized for $\p{x} = -\p{e_1}$. Therefore the ratio for any $\p{x}, \p{y} \in B_d(\p{0}_d, 1)$ and $\p{c} = 2d~\p{e_1}$ is at most
\begin{align*}
 \frac{a_c^{(\p{x}) } }{a_c^{(\p{y})} } &\le \frac{  7d-2  }{ d+2 } \le 7
\end{align*}
The ratio for all other vertices can be computed in a similar fashion and is bounded by the same quantity. 
\end{proof}
\begin{remark}
Note that in our analysis, we show that the privacy holds for in the worst case for \emph{any} two gradients vectors $\p{g_1}, \p{g_2} \in B_d(\p{0}_d, 1)$ and not just gradients of neighboring datasets. The analysis can be tightened If we make certain assumptions about the distribution of the data points. 
\end{remark}

\paragraph{(2) Hadamard Scheme: }
We now propose another quantization scheme with same communication cost, but provides better privacy guarantees.  This quantization scheme is similar to the one presented in Section~\ref{sec:derandomRM} and is based on the columns of a Hadamard matrix (see Definition~\ref{def:hadamard}) formed using the Sylvester construction.  

Let us assume that $d+1$ is a power of $2$ \ie $d+1 = 2^p$ for some $p\ge1$.  
For any $i \in [d+1]$, let $\p{h_i} \in \R^d$ denote the $i$-th column of $H_p$ with the first coordinate punctured. Consider the following set of $d+1$ points obtained from the punctured columns of $H_p$: 
\begin{align*}
C_{H} = \{2\sqrt{d}~\p{h_i}  \mid i \in [d+1] \}
\end{align*}
The quantization scheme $Q_{C_{H}}$ can be implemented in linear time since computing the probabilities requires computing a matrix vector product, 
\begin{align*}
(d+1) \cdot \begin{bmatrix} a_1 & \cdots & a_{d+1} \end{bmatrix}^T =  H_p^T \begin{bmatrix} 1& \p{v}^T / (2\sqrt{d}) \end{bmatrix}^T
\end{align*}
that has closed form solution for each $a_i$ as: 
\begin{align}\label{eq:hadamard_sol}
a_i = \frac{1}{d+1} \cdot \left(1 + \frac{\p{h_i}^T \p{v}}{2 \sqrt{d}} \right)
\end{align}

\begin{proposition}\label{prop:hadamard}
For any $\p{v} \in B_d(\p{0}_d,1)$, let $\p{\hat{v}} := Q_{C_{H}}(\p{v})$.  Then, 
$\E[\p{\hat{v}}] = \p{v}$ and $\E \left[ \| \p{v} - \p{\hat{v}} \|_2^2 \right] = O(d^2).$ 
Moreover, $Q_{C_H}$ is $\eps$-DP for any $\eps > \log (1+\sqrt{2})$. 
\end{proposition}
\begin{proof}
First, we show that $C_H$ satisfies Condition~\eqref{cond:hull} with $R=2d$. The fact that $\conv(C_H) \subset B_d(\p{0}_d,2d)$ is trivial and follows since every point in $C_H$ is contained in $B_d(\p{0}_d, 2d)$. 

To show that $B_d(\p{0}_d,1) \subset C_H$, consider any $\p{x} \in B_d(\p{0}_d,1)$, and the closed form solution for the coefficients $a_i$ given by Equation~\eqref{eq:hadamard_sol}. We now show that these coefficients indeed give a convex combination. Note that 
 $a_i := \frac{1}{d+1} \left( 1 +  \frac{\p{c}^T\p{x}}{4d} \right) \ge 0$. This holds since $\p{c}^T\p{x} \ge \|\p{c}\|\|\p{x}\| \ge -2d$.  Moreover, from the property of Hadamard matrices, 
\[
\sum_{i=1}^{d+1} a_i = \frac{1}{d+1}  \begin{bmatrix} 1&\ldots & 1\end{bmatrix} H_p^T \begin{bmatrix} 1\\ \frac{\p{x}}{2\sqrt{d}} \end{bmatrix} = 1.
\]
The last equality follows from the following property of the Hadamard matrices that can be proved using induction. 
\[
\begin{bmatrix} 1&\ldots & 1\end{bmatrix} H_p^T = \begin{bmatrix} 2^p& 0 &\ldots & 0\end{bmatrix}.
\]
Therefore, any $\p{x} \in B_d(\p{0}_d,1)$ can be expressed as a convex combination of the points in $C_H$, \ie, $\p{x} = \sum_{i =1}^{d+1} a_i \p{c_i}$, for $\p{c_i} \in C_H$.

\paragraph{Privacy: } We now show that the quantization scheme is $\eps$-differentially private for any $\eps > 0.4$. From the definition of $\eps$-DP, it is sufficient to show that for any $\p{x}, \p{y} \in B_d(\p{0}_d,1)$, and any $\p{c} \in C_H$,
\[
 \frac{\Pr[Q_{C_H}(\p{x}) = \p{c}]}{\Pr[Q_{C_S}(\p{y}) = \p{c}]}  \le 1+\sqrt{2}
\]
Since $\p{x}, \p{y} \in \conv(C_S)$, we can express them as the convex combination of points in $C_H$. Let $\p{x} = \sum_{\p{c} \in C_H} a_c^{(\p{x})}$. Similarly, let $\p{y} = \sum_{\p{c} \in C_H} a_c^{(\p{y})}$. Then, from the construction of the quantization function $Q_{C_H}$, we know that 
\begin{align}\label{eq:had1}
\frac{\Pr[Q_{C_H}(\p{x}) = \p{c}]}{\Pr[Q_{C_H}(\p{y}) = \p{c}]} = \frac{a_c^{(\p{x}) } }{a_c^{(\p{y})} }. 
\end{align}
From the closed form solution in Equation~\eqref{eq:hadamard_sol}, we know that for any $\p{x} \in \conv(C_H)$, the coefficient of $\p{c}$ in the convex combination of $\p{x}$ is given by $a_c^{(\p{x})} = \frac{1}{d+1} \left( 1 +  \frac{\p{c}^T\p{x}}{4d} \right)$. 
Plugging this in Equation~\eqref{eq:had1}, we get 
\begin{align}\label{eq:had2}
\frac{\Pr[Q_{C_H}(\p{x}) = \p{c}]}{\Pr[Q_{C_H}(\p{y}) = \p{c}]} &= \frac{a_c^{(\p{x}) } }{a_c^{(\p{y})} } 
= \frac{1 +  \frac{\p{c}^T\p{x}}{4d} }{ 1 +  \frac{\p{c}^T\p{y}}{4d}} = 1+ \frac{\frac{\p{c}^T(\p{x}-\p{y})}{4d}}{ 1 +  \frac{\p{c}^T\p{y}}{4d}}  \\
& \leq 1 + \frac{\| \p{c} \|_2\|\p{x} - \p{y} \|_2}{4d - \|\p{c} \|_2}  \qquad \mbox{ for }   \p{y} = -\frac{\p{c}}{\|\p{c}\|_2}\\
& \leq 1+ \frac{2\sqrt{2} d}{4d-2d } \
\intertext{(since $  \| \p{x} - \p{y} \|_2 \le \sqrt{2} \mbox{ and } \| \p{c} \|_2 = 2d.)$ }
& =1 + \sqrt{2}
\end{align}
This concludes the proof of Proposition~\ref{prop:hadamard}. 
\end{proof}

Finally, we remark that even though the point set $C_{cp}$  in the cross-polytope scheme (in Section~\ref{sec:cp}) has more than $d+1$ points, it still gives us $\eps$-DP for any $\eps > O(\log d)$.  Note that differential privacy for such large parameters is also of interest to the community \cite{acharya2019hadamard}.
\begin{proposition}\label{prop:cp:priv}
Let $\tilde{C}_{cp}$ be the set of point in the cross-polytope point set scaled by a factor of $2$. 
Let $\tilde{C}_{cp} = \{ \pm 2\sqrt{d} e_i |  i \in [d] \}$, then 
$Q_{\tilde{C}_{cp}}$ is $\eps$-DP for any $\eps > \log d$.
\end{proposition}

\begin{proof}
The fact that $Q_{\tilde{C}_{cp}}$ satisfies Condition~\eqref{cond:hull} with $R = 2\sqrt{d}$ follows from the proof of Proposition~\ref{prop:crosspolytope}. For any $v \in \R^d$, we can compute  the convex combinations as 
\begin{align} \label{eq:solution_cp}
a_i =
\left\{
	\begin{array}{ll}
		\frac{v_i}{2\sqrt{d}} + \frac{\gamma}{2d}  & \mbox{if } v_i > 0  \mbox{ and } i \le d\\[3pt]
		- \frac{v_i}{2\sqrt{d}} + \frac{\gamma}{2d}  & \mbox{if } v_i \le 0  \mbox{ and } i > d\\[3pt]
		\frac{\gamma}{2d}  & \mbox{otherwise } \\[3pt]
	\end{array}
\right.
\end{align}
where, $\gamma := 1 -\frac{\|\p{v}\|_1}{2\sqrt{d}}$, is a non-negative quantity for every $\p{v} \in B_d(\p{0}_d,1)$.

To prove the privacy guarantees of this scheme, we first state a few observations: 
\begin{itemize}
\item Since $\|\p{v}\|_1 \in [-\sqrt{d}, \sqrt{d}]$, the quantity $\gamma \in [1/2, 3/2]$.
\item For any coordinate $i \in [d]$, if $x_i > 0$, then the coefficients $a_i = \frac{|x_i|}{2\sqrt{d}} + \frac{\gamma}{2d}$, and $a_{d+i} = \frac{\gamma}{2d} $.
\item Similarly, if $x_i \le 0$, then the coefficients $a_i = \frac{\gamma}{2d}$, and $a_{d+i} = \frac{|x_i|}{2\sqrt{d}} + \frac{\gamma}{2d}$.
\item For any $\p{x} \in B_d(\p{0}_d,1)$, $x_i \in [-1, 1]$
\end{itemize}

Let $\p{x}, \p{y} \in B_d(\p{0}_d,1)$ and for any $\p{c} \in \tilde{C}_{cp}$, we need to upper bound the following quantity to prove the privacy guarantees of the scheme:
\begin{align*}
p_c:= \frac{\Pr[ Q_{\tilde{C}_{cp}}(\p{x}) = \p{c} ]}{\Pr[ Q_{\tilde{C}_{cp}}(\p{y}) = \p{c} ]}
\end{align*}

Note that it is sufficient to consider only one of the points $\p{c} = 2\sqrt{d} \p{e_j}$  in the following four scenarios: 
\begin{enumerate}
\item $x_i > 0, y_i > 0$, then 
\begin{align*}
p_{c} &= \frac{ \frac{|x_i|}{2\sqrt{d}} + \frac{\gamma_x}{2d}  }{ \frac{|y_i|}{2\sqrt{d}} + \frac{\gamma_y}{2d} }
\le \frac{ \frac{1}{2\sqrt{d}} + \frac{3}{4d}  }{\frac{1}{4d} } 
\le O(\sqrt{d}).
\end{align*}

\item $x_i > 0, y_i \le 0$, then 
\begin{align*}
p_{c} &= \frac{ \frac{|x_i|}{2\sqrt{d}} + \frac{\gamma_x}{2d}  }{ \frac{\gamma_y}{2d} }
\le  \frac{ \frac{1}{2\sqrt{d}} + \frac{3}{4d} }{ \frac{1}{4d} } 
\le O(\sqrt{d}).
\end{align*}

\item $x_i  \le  0, y_i > 0$, then 
\begin{align*}
p_{c} &= \frac{ \frac{\gamma_x}{2d} }{ \frac{|y_i|}{2\sqrt{d}} + \frac{\gamma_y}{2d} }
= \frac{ \frac{3}{4d} }{\frac{1}{4d} }
\le 3
\end{align*}

\item $x_i  \le  0, y_i  \le  0$, then 
\begin{align*}
p_{c} &= \frac{ \frac{\gamma_x}{2d} }{ \frac{\gamma_y}{2d} }
= \frac{ \frac{3}{4d} }{\frac{1}{4d}} 
\le 3.
\end{align*}
\end{enumerate}
Therefore, the privacy guarantees hold for any $\epsilon > O(\log d )$. 
\end{proof}

We now show a Randomized Response (RR) scheme that can be used on top of any of our quantization schemes to achieve privacy. This scheme incurs the same communication as the original quantizer, however, the price of privacy is paid by factor of $d$ increase in the variance. We also propose a weaker version using \rappor, that incurs a higher communication cost depending on the point set of choice.

\subsection{Randomized Response}\label{sec:rr}
We present a Randomized Response (RR) mechanism, introduced by  \cite{randomresponse}, that can be used over the output of $Q_C$ to make it $\eps$-DP (for any $\eps>0$). This modified scheme retains the original communication cost of $Q_C$, but the cost for privacy is paid by a factor of $O(d)$ in the variance term.
 
Recall that the quantization scheme described in Section \ref{sec:quant}, $Q_C(\p{v})$, takes a vector $\p{v} \in B_d(\p{0}_d,1)$ and returns a point $\p{c_i} \in C$. The \RR scheme takes the output of $Q_C(\p{v})$ and returns a another random vector from $C$. 

For any $\eps > 0$, define $p := p(\eps) =\frac{e^{\epsilon}}{e^{\epsilon} +|C|-1}$ and $q := \frac{1 -p}{|C| - 1} =  \frac{1}{e^{\epsilon} +|C|-1}$. 
We define the private quantization of a vector $\p{v} \in B_d(\p{0}_d,1)$ as
\begin{align*}
 \hat{\p{v}} = PQ_{C,\epsilon}(\p{v})= \frac{1}{p-q}  \sum_{i=1}^{|C|}(\mathbf{1}_{\{\p{y} = \p{c_i} \} } - q)\p{c_i}, 
\end{align*} 
where, $\mathbf{1}_{\{\p{y} = \p{c_i} \} }$ is an indicator of the event $\p{y} = \p{c_i}$ and $\p{y} := \RR_p(Q_C(\p{v}), C) $ is defined as 
\begin{align*}
\RR_p(Q_C(\p{v}), C) =
\left\{
	\begin{array}{ll}
		Q_C(\p{v}) & \mbox{w.p. } p \\
		\p{z} \in C \setminus \{ Q_C(\p{v}) \} & \mbox{w.p. } q
	\end{array}
\right.
\end{align*}

We claim that the quantization scheme $PQ_{C,\eps}$ is $\eps$-differentially private.
\begin{theorem}\label{thm:quantrr}
Let $C \subset \R^d$ be any point set satisfying Condition~\eqref{cond:hull}. For any $\eps > 0$, let $p = \frac{e^{\epsilon}}{e^{\epsilon} +|C|-1}$ and $q = \frac{1}{e^{\epsilon} +|C|-1}$. 
For any $\p{v} \in B_d(\p{0}_d,1)$, let
 $\hat{\p{v}} = PQ_{C,\epsilon}(\p{v})= \frac{1}{p-q}  \sum_{i=1}^{|C|}(\mathbf{1}_{\{\p{y} = \p{c_i} \} } - q)\p{c_i}$, 
where,  $\p{y}:=  \RR_p( Q_C(\p{v}) ,C)$.  Then,
$\E[\p{\hat{v}}] = \p{v}$  and $\E\left[ \| \p{v} - \p{\hat{v}} \|_2^2\right] = O(|C|R^2)$, where the expectation is taken over the randomness in both $Q_C$ and $\RR_p$. Moreover, the scheme is $\eps$-differentially private.
\end{theorem}

\begin{proof}
First we show that $\hat{\p{v}} = PQ_{C,\epsilon}(\p{v})= \frac{1}{p-q}  \sum_{i=1}^{|C|}(\mathbf{1}_{\{\p{y} = \p{c_i} \} } - q)\p{c_i}$ is an unbiased estimator of $v$. From linearity of expectations, we have
\begin{align}\label{eq:rr1}
\E[\hat{\p{v}}] &= \frac{1}{p-q}  \sum_{i=1}^{|C|}( \Pr[\p{y} = \p{c_i}] - q)\p{c_i}, 
\end{align}
where, the expectation is taken over the randomness of both the quantization and \RR scheme.
Recall that 
\[
 \p{y} := \RR_p(Q_C(\p{v}) ,C) \in C, 
\]
where $p =\frac{e^{\epsilon}}{e^{\epsilon} +|C|-1}$. Therefore, 
\begin{align*}
    \Pr(\p{y} = \p{c_i}) &= \sum_{j = 1}^{|C|} \Pr[ \p{y}  =  \p{c_i} | Q_C(\p{v}) = \p{c_j} ] \cdot \Pr[Q_C(\p{v}) = \p{c_j} ] \\
    &=( p-q) a_i + q.
\end{align*}
Therefore $\E[\hat{\p{v}}] = \frac{1}{p-q}  \sum_{i=1}^{|C|} (p-q) a_i \p{c_i} = \p{v}$. 

Now we bound the variance of the estimator 
\begin{align*}&
\E[\| \p{v} -\hat{\p{v}} \|^2 ] =\E\left[ \|\sum_{i=1}^{|C|}\left(\frac{1}{p-q} (\mathbf{1}_{\{\p{y}= \p{c_i} \}} - q)  -a_i\right)\p{c_i} \|^2 \right]\\
 & \leq  \sum_{i=1}^{|C|}\E\left[ \left(\frac{1}{p-q} (\mathbf{1}_{\{\p{y} = \p{c_i} \}} - q)  -a_i \right)^2 \|\p{c_i}  \| ^2\right] \\
 & =\sum_{i=1}^{|C|} \textsc{Var}\left[ \left(\frac{1}{p-q} (\mathbf{1}_{\{ \p{y} = \p{c_i} \}} - q) \right) \|\p{c_i}  \| ^2 \right] \\
 & =\left( \frac{1}{p-q} \right)^2 \sum_{i=1}^{|C|} \textsc{Var}(\mathbf{1}_{\{\p{y} = \p{c_i} \}} ) \|\p{c_i}  \| ^2 \\
 & =O( |C|R^2), 
\end{align*}
since $\|c_i\|^2 \le R^2$ and $\textsc{Var}(\mathbf{1}_{\{ \p{y} = \p{c_i} \}}) \le 1/4$ .

\paragraph{Privacy} Now we show that our scheme is $\epsilon$ differentially private where $\epsilon$ is the input parameter to the \RR algorithm. For any two points  $\p{v},\p{w} \in B_d(\p{0}_d,1)$, 
\begin{align}
& \frac{PQ_{C, \eps}(\p{v})=y}{PQ_{C, \eps}(\p{w})=y}  = \frac{\sum_{i=1}^{|C| } \Pr(y| Q_C(\p{v})= \p{c_i} )\Pr( Q_C(\p{v})= \p{c_i})}{\sum_{j=1}^{|C| } \Pr(y | Q_C(\p{w})= \p{c_j} )\Pr( Q_C(\p{w})= \p{c_j})}\\
& \leq \frac{\max_{i}\Pr(y| Q_C(\p{v})= \p{c_i}) \sum_{i=1}^{|C| } \Pr( Q_C(\p{v})= \p{c_i})}{\min_{j}\Pr(y | Q_C(\p{w})= \p{c_j} )\sum_{i=1}^{|C| } \Pr( Q_C(\p{w})= \p{c_j})}\\
&=\frac{\max_{i} \Pr(y| Q_C(\p{v})= \p{c_i}) }{\min_{j}\Pr(y | Q_C(\p{w})= \p{c_j} )} \leq e^{\epsilon} \label{eq:rr4}
\end{align}
we are using the following privacy property of Randomized Rounding \cite{randomresponse} mechanism in Equation \eqref{eq:rr4}
\begin{align*}
\sup_{i,j} \frac{\Pr(y| Q_C(\p{v})= \p{c_i})}{\Pr(y | Q_C(\p{w})=\p{c_j})} \leq e^{\epsilon}\quad \forall\p{v},\p{w}
\end{align*}
\end{proof}

\subsection{Privacy using \rappor  }\label{sec:rappor}
In this section, we present an alternate mechanism to make the quantization scheme 
$\eps$-DP (for any $\eps>0$). 
The main idea is to use the \rappor mechanism (\cite{rappor}) over a 1-hot encoding of the indices of vertices in $C$.  
Though in doing so, we have to tradeoff on the communication a bit. 
Instead of sending $\log |C|$ bits, this scheme now requires one to send $O(|C|)$ bits to achieve privacy.

Recall that the quantization scheme  described in Section \ref{sec:quant}, $Q_C(\p{v})$, takes a vector $\p{v} \in B_d(\p{0}_d,1)$ and returns a point $\p{c_i}$ in $C$. 
We can interpret the output as the bit string $\p{b} \in \{0,1\}^{|C|}$ which is the indicator of the point $\p{c_i}$ in $C$ (according to some fixed arbitrary ordering of $C$). 
Note that this is essentially the $1$-hot encoding of $\p{c_i}$.  
In the RAPPOR scheme each bit of the 1-hot bit string $\p{b}$ is flipped independently with probability $p:= p(\eps)= \frac{1}{(e^{\eps/2} + 1)}$.

For any $\eps>0$, let $p =\frac{1}{(e^{\eps/2} + 1)}$. Define, the private quantization of a vector $\p{v} \in B_d(\p{0}_d,1)$ as
\[\hat{\p{v}} := PQ_{C, \eps}(\p{v}) = \frac{1}{(1-2p)}\sum_{j =1}^{|C|} \left(y_j - p \right)
\p{c_j}\] where, 
$\p{y} := \rappor_p(\onehot( Q_C(\p{v}) ,C)) \in \{0,1\}^{|C|}$.

We claim that the quantization scheme $PQ_{C,\eps}$ is $\eps$-differentially private. Moreover, adding the noise over the \onehot encoding maintains the unbiasedness of the gradient estimate but incurs a factor of $|C|$ in variance term while the communication cost is $O(|C|)$.

\begin{theorem}\label{thm:quantrappor}
Let $C \subset \R^d$ be any point set satisfying Condition~\eqref{cond:hull}. For any $\eps > 0$, let $p = \frac{1}{(e^{\eps/2} + 1)}$. 
For any $\p{v} \in B_d(\p{0}_d,1)$, let
 $\p{\hat{v}} := \frac{1}{1-2p}\sum_{j =1}^{|C|} \left(y_j - p \right)\p{c_j}$, 
where,  $\p{y}:=  \rappor_p(\onehot( Q_C(\p{v}) ,C))$.  Then,
$\E[\p{\hat{v}}] = \p{v}$  and $\E\left[ \| \p{v} - \p{\hat{v}} \|_2^2\right] = O(|C| R^2)$. 
Moreover, the scheme is $\eps$-differentially private.
\end{theorem}

\begin{proof}
First we show that $\hat{\p{v}} = \frac{1}{(1-2p)}\sum_{j =1}^{|C|} \left(y_j - p \right)\p{c_j}$ is an unbiased estimator of $v$. From linearity of expectations, we have
\begin{align}\label{eq:rappor1}
\E[\hat{\p{v}}] &= \frac{1}{(1-2p)}\sum_{j =1}^{|C|} \left( \E[y_j] - p \right)\p{c_j}, 
\end{align}
where, the expectation is taken over the randomness of both the quantization and RAPPOR scheme.

Recall that 
\[
 \p{y} := \rappor_p(\onehot( Q_C(\p{v}) ,C)) \in \{0,1\}^{|C|}.
\]
Each entry of the vector $\p{y}$ is an  independent binary random variable and 
\begin{align}\label{eq:rappor2}
\E[y_j] & = \Pr(y_j=1) = \sum_{i=1}^{|C| } \Pr(y_j ,Q_C(\p{v})= \p{c_i} ) \nonumber \\
& = \sum_{i=1}^{|C| } Pr(y_j | Q_C(\p{v})= \p{c_i} )Pr( Q_C(\p{v})= \p{c_i}) \nonumber \\
& =Pr(y_j | Q_C(\p{v})= \p{c_j} )Pr( Q_C(\p{v})= \p{c_j})\nonumber \\
&+\sum_{i\neq j}^{|C| } Pr(y_j | Q_C(\p{v})= \p{c_i} )Pr( Q_C(\p{v})= \p{c_i}) \nonumber \\
&=(1-p)a_j + p(1-a_j)= p+(1-2p)a_j.
\end{align} 
Plugging Equation~\eqref{eq:rappor2} in Equation~\eqref{eq:rappor1} , we get 
\begin{align}
\E(\hat{\p{v}}) = \frac{1}{(1-2p)}\sum_{j =1}^{|C|} \left( p+(1-2p)a_j - p \right)\p{c_j} =\sum_{j =1}^{|C|}a_j\p{c_j}=\p{v}
\end{align}
Now we show a bound on the variance of the estimate 
\begin{align}
\E\left[ \| \p{v} - \p{\hat{v}} \|_2^2\right]& = \E\left[ \| \sum_{j =1}^{|C|}a_j\p{c_j} - \frac{1}{(1-2p)}\sum_{j =1}^{|C|} \left(y_j - p \right)\p{c_j}\|_2^2\right] \\
&=\sum_{j =1}^{|C|}\E \left( a_j -  \frac{(y_j-p)}{(1-2p)} \right)^2\left|\p{c_j} \right|^2\\
\intertext{( all the cross terms are $0$ as they are mutually independent and $\E \left( a_j -  \frac{y_j-p}{1-2p} \right)=0$)} 
&= \sum_{j =1}^{|C|} \Var\left(\frac{y_j-p}{1-2p} \right)\left|\p{c_j} \right|^2 \\ 
&=\left(\frac{1}{1-2p}\right)^2  \sum_{j =1}^{|C|} \Var(y_j) \left|\p{c_j} \right|^2= O(|C| R^2) \label{eq:rappor3}
\end{align}
Equation \eqref{eq:rappor3} comes form the fact that $y_j$ is a binary random variable and $Var(y_j)=Pr(y_j)(1-Pr(y_j)) \leq \frac{1}{4}$ and $\left|\p{c_j} \right|^2 \leq R^2$.

\paragraph{Privacy} Now we show that our scheme is $\epsilon$ differentially private where $\epsilon$ is the input parameter to the RAPPOR algorithm. For any two points  $\p{v},\p{w} \in B_d(\p{0}_d,1)$, 
\begin{align}
& \frac{PQ_{C, \eps}(\p{v})=y}{PQ_{C, \eps}(\p{w})=y}  = \frac{\sum_{i=1}^{|C| } \Pr(y| Q_C(\p{v})= \p{c_i} )\Pr( Q_C(\p{v})= \p{c_i})}{\sum_{j=1}^{|C| } \Pr(y | Q_C(\p{w})= \p{c_j} )\Pr( Q_C(\p{w})= \p{c_j})}\\
& \leq \frac{ \max_{i}\Pr(y| Q_C(\p{v})= \p{c_i}) \sum_{i=1}^{|C| } \Pr( Q_C(\p{v})= \p{c_i})}{\min_{j}\Pr(y | Q_C(\p{w})= \p{c_j} )\sum_{i=1}^{|C| } \Pr( Q_C(\p{w})= \p{c_j})}\\
&=\frac{ \max_{i} \Pr(y| Q_C(\p{v})= \p{c_i}) }{\min_{j}\Pr(y | Q_C(\p{w})= \p{c_j} )} \leq e^{\epsilon} \label{eq:rapor4}
\end{align}
By the privacy property of  RAPPOR \cite{rappor} mechanism , we are using  the following fact in equation \eqref{eq:rapor4}
\begin{align*}
\sup_{i,j} \frac{\Pr(y| Q_C(\p{v})= \p{c_i})}{\Pr(y | Q_C(\p{w})=\p{c_j})} \leq e^{\epsilon}\quad \forall\p{v},\p{w}
\end{align*}
\paragraph{Communication :} Now we show that for the RAPPOR based scheme the expected communication is   linear in $|C|$. Say $y $ is the output when RAPPOR is applied to one hot encoded binary string. Without loss of generality say the the bit string is $\p{e_i}$. The output $y$ is generated as follows
\[
\Pr(y_j=1) = \begin{dcases*}
p & if  $j\neq i$\\
(1-p) & if $j=i$
\end{dcases*}
\]
So the expected sparsity ($l_0 $ norm) of the output is 
\begin{align*}
\E[\|y\|_0]&= \sum_{i}^{|C|}y_i = (|C|-1)p+(1-p)\\
&=|C|p + (1-2p) = O(|C|)
\end{align*}
\end{proof} 


\section{Experiments}\label{sec:expe}
We use our gradient quantization scheme to train a fully connected ReLU activated network with $1000$ hidden nodes using the MNIST~\cite{lecun} and the Fashion MNIST~\cite{xiao2017/online} dataset  (60000 data points with 10 classes for each).  We use the \emph{cross-entropy loss} function for the training the neural network with a total of $d=795010$ parameters. 

The dataset is divided equally among $100$ workers. Each worker computes the local gradients and communicates the quantized gradient to the master which then aggregates and send the updated parameters.
 We plot the error at each iteration (Figure~\ref{comp}) and compare our results with QSGD quantization.  

We use vqSGD with cross polytope scheme, $Q_{C_{cp}}$, along with the variance reduction technique with repetition parameter $s=100$. Therefore, each local machine sends about $2060 = 100\cdot \log(2d)$ bits per iteration whereas, QSGD requires $3825.05$  bits for MNIST and $2266.79$ bits for Fashion MNIST, of communication per iteration per machine (computed by averaging over the total bits of communication over $50$ iterations) to communicate the quantized gradient. Our results indicate that vqSGD converges at a similar rate to QSGD while communicating much lesser bits. 

We also our vqSGD with the cross polytope scheme, $Q_{C_{cp}}$, to train a ReLU network with $4000$ hidden nodes using the CIFAR 10 dataset~\cite{krizhevsky2009learning}. This dataset also has 10 classes, every other set up is same except now we have $d=12332010$ parameters.

The dataset is again equally divided among 100 users. Using vqSGD, each machine send $2455$ bits per iteration using the variance reduction scheme. On the other hand, for QSGD,  the number of bits per machine per iteration is $4096.9$ (computed by averaging over the total bits of communication over $50$ iterations). As is evident from the plot in Figure~\ref{comp}, vqSGD communicates lesser number of bits to achieve similar performance.

\begin{figure}[h!]
  \centering
     \subfloat[MNIST]{{\includegraphics[width=0.32\textwidth]{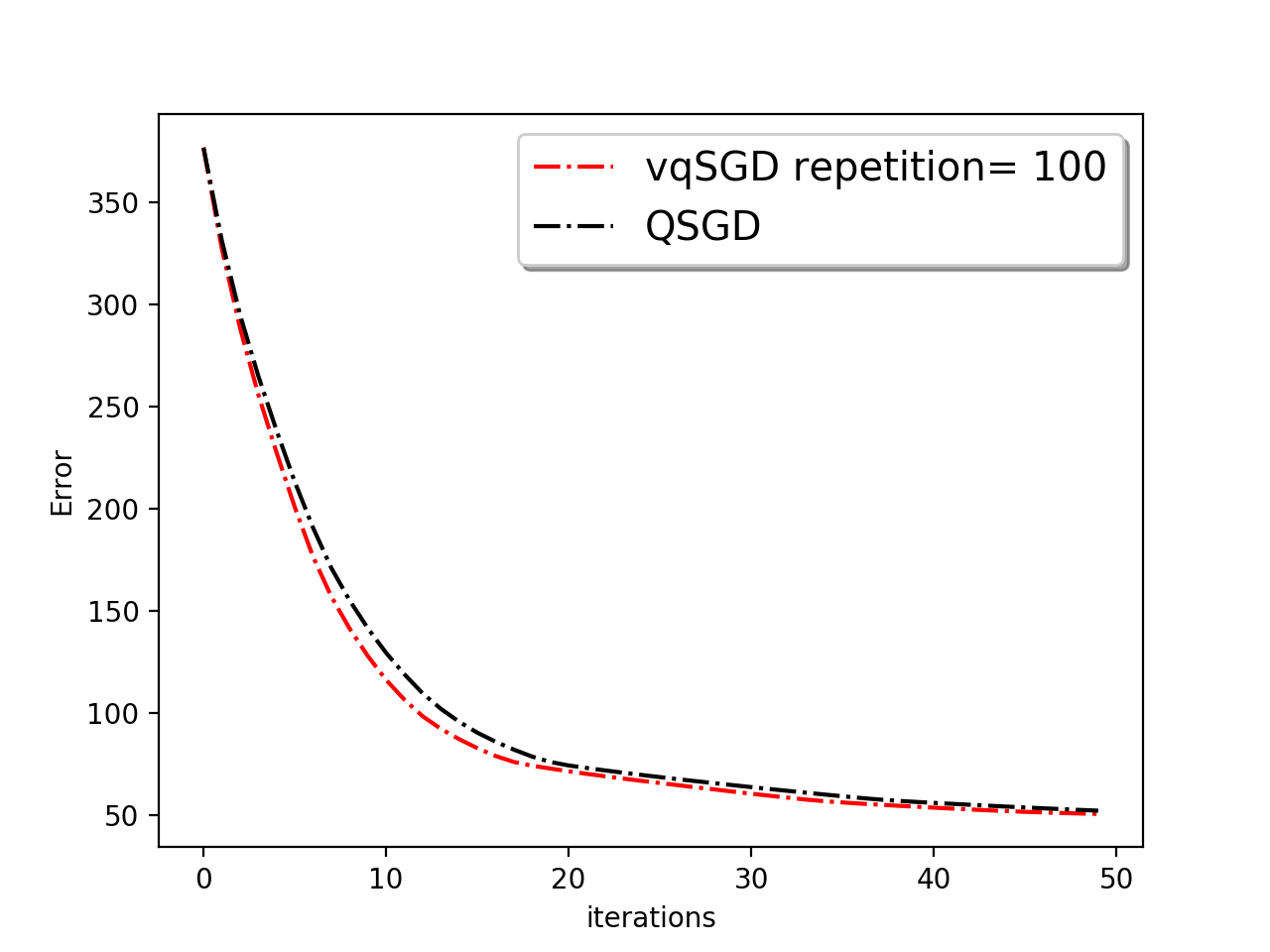} }}
     \subfloat[Fashion MNIST]{{\includegraphics[width=0.32\textwidth]{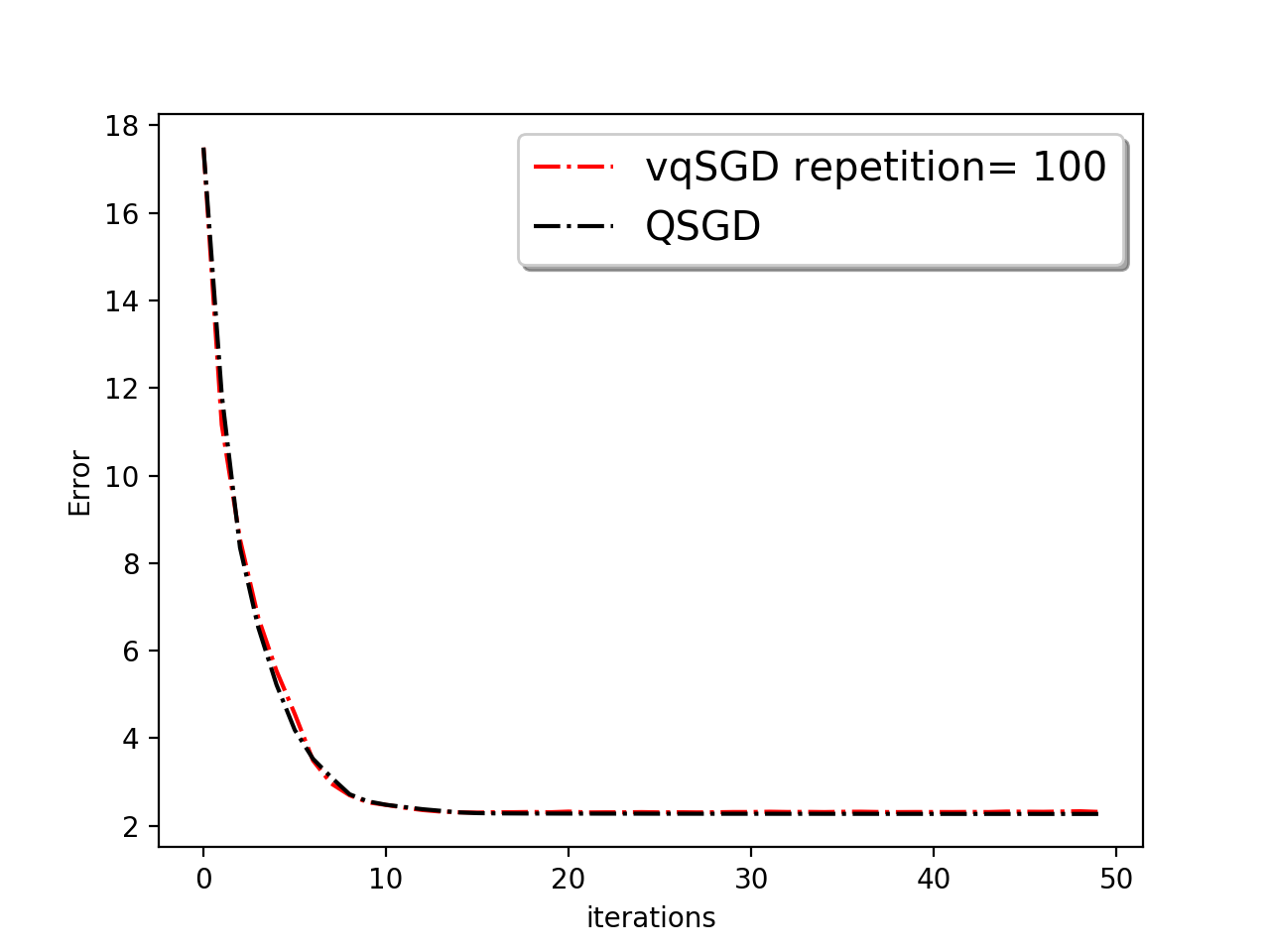} }}
      \subfloat[CIFAR10]{{\includegraphics[width=0.32\textwidth]{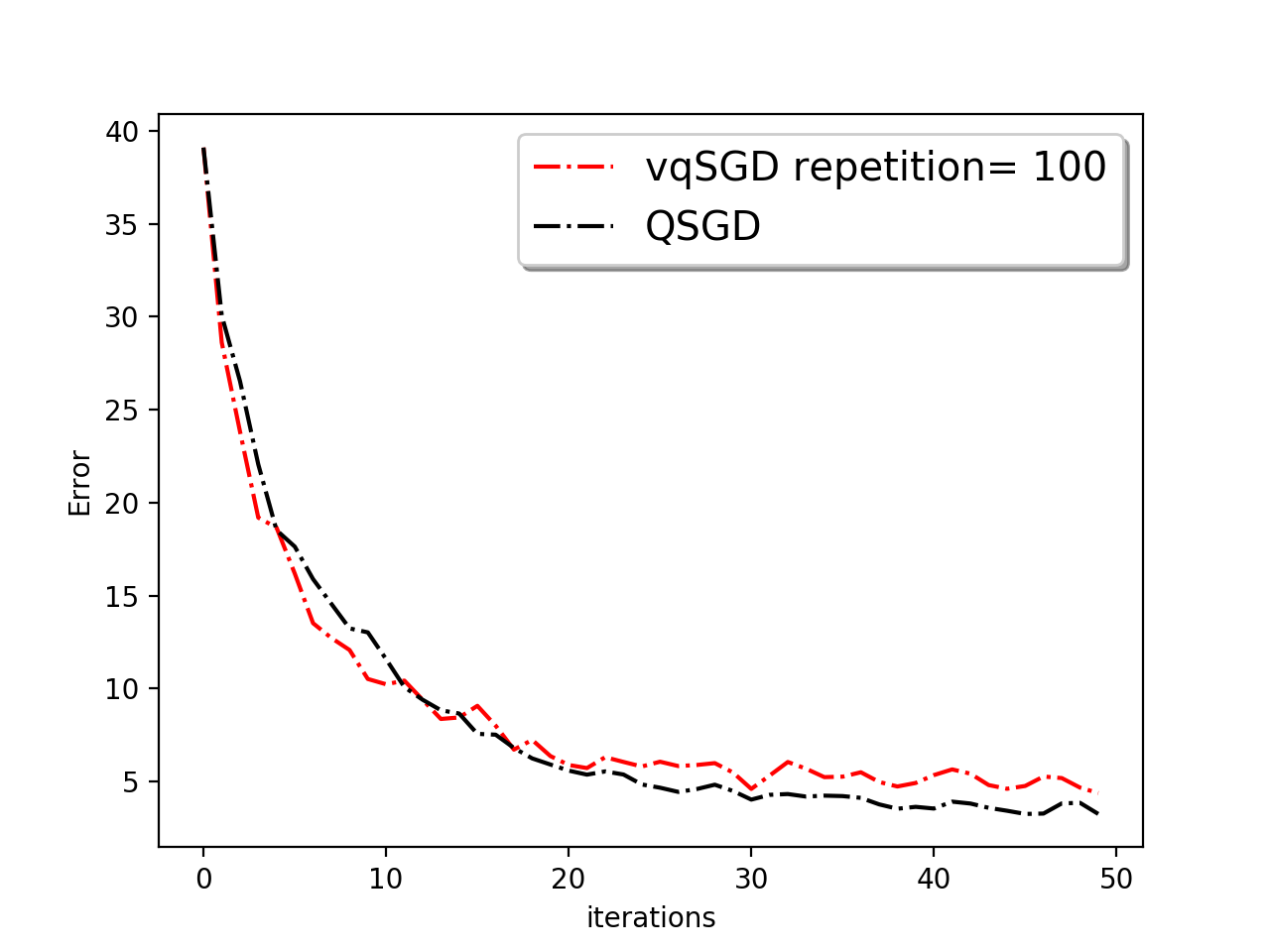} }}
    \caption{Convergence for fully connected ReLU network compared to QSGD  }
    \label{comp}
\end{figure}

Further we experimentally 
show the  performance of vqSGD using the cross polytope $Q_{cp}$, to solve the least squares problem and logistic regression for binary classification.

\textbf{Least Squares:} In the least square problem, we solve for $\p{\theta^*} =    \argmin_{\p{\theta}} \|A \p{\theta} -\p{b} \|_2^2$, where the  matrix $A\in \mathbb{R}^{n\times d}$ and $\p{\theta^*} \in \mathbb{R}^d$ are generated by sampling each entry from $\mathcal{N}(0,1)$ and we set $\p{b}=A \p{\theta^*}$.  

In order to show the performance of vqSGD, we simulate the iterations of distributed SGD with $n=10000$ data samples distributed equally among $N=500$ worker nodes. In every iteration of SGD, each worker node computes the local gradient on individual data batch and communicates the quantized version of the local gradient to the parameter server. 
The parameter server on receiving all the quantized gradients averages them and broadcasts the updated model to all the workers. 
The convergence of SGD is measured by the error term $\| \p{\theta^*} -\p{\theta_t} \|_2$, where $\p{\theta_t}$ is the computed parameter at the end of $t$-th iteration of distributed SGD.

We compare the convergence of the least square problem for $d=100,200,500$ against the state-of-the-art quantization schemes  - DME \cite{dme} and QSGD \cite{qsgd}. The results are presented in  Figure~\ref{Compare}.  

\begin{figure*}[h!]%
    \centering
    \subfloat[d=100]{{\includegraphics[width=0.30\textwidth]{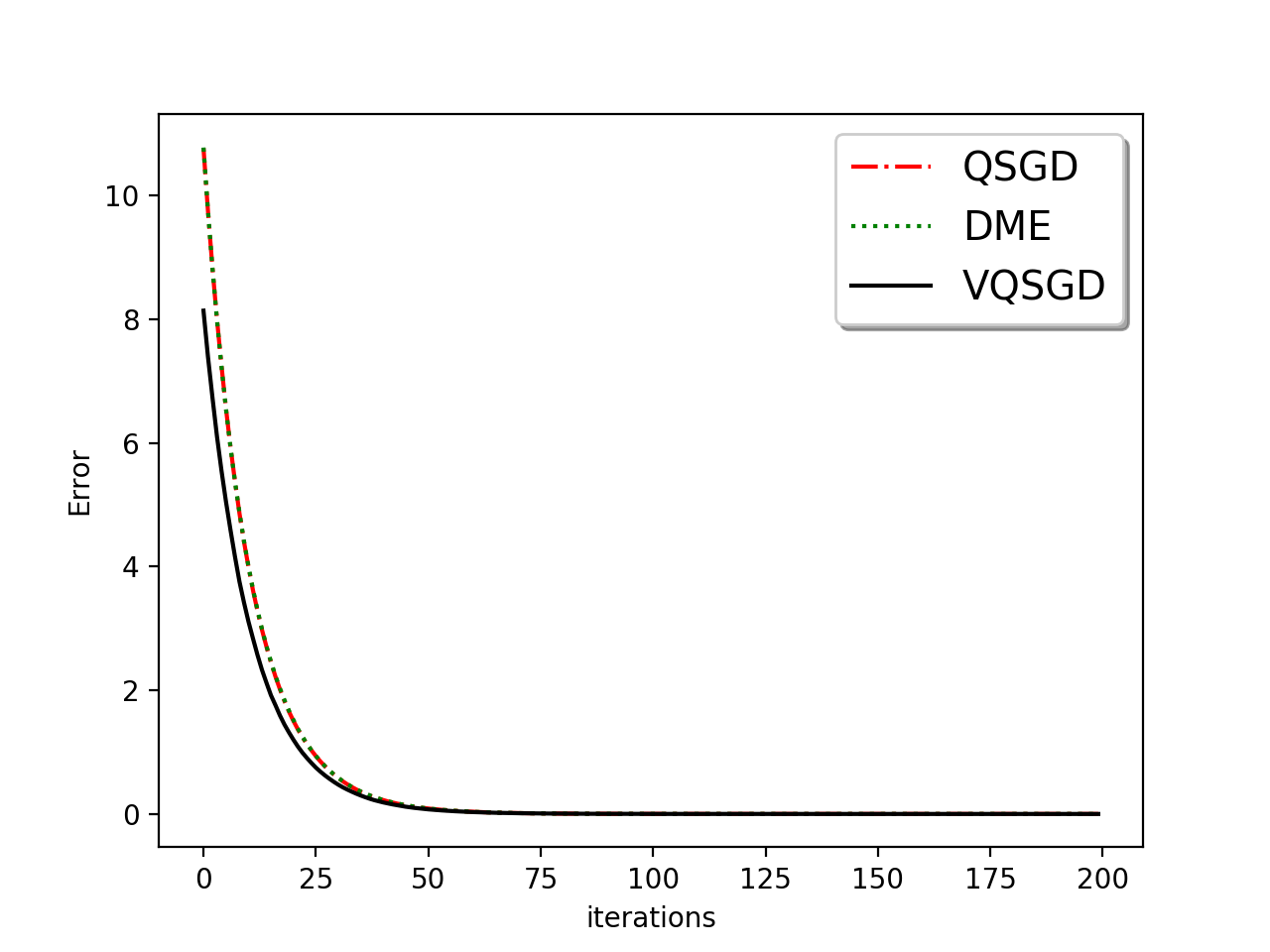} }}%
    \quad
    \subfloat[d=200]{{\includegraphics[width=0.30\textwidth]{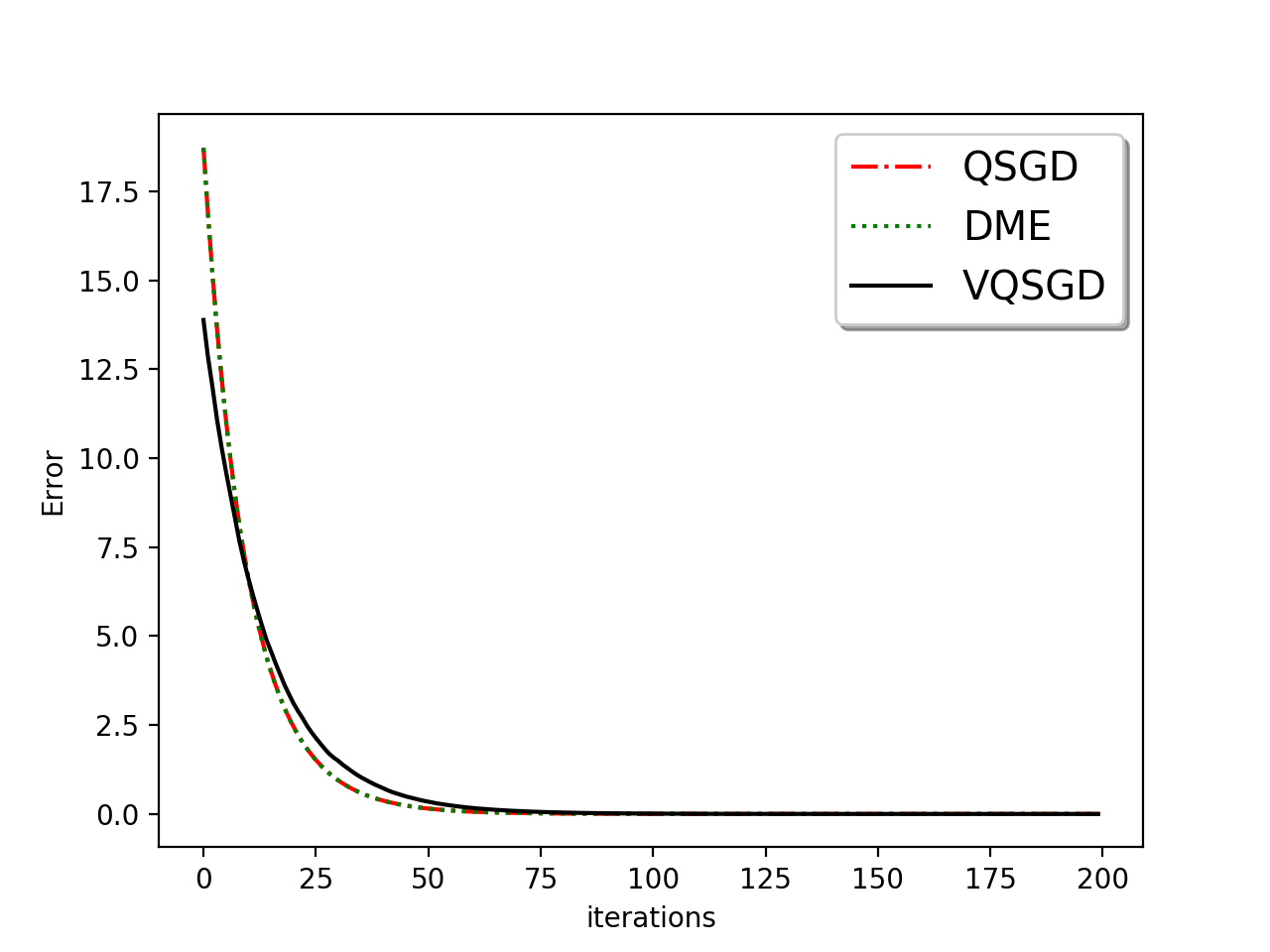} }}%
     \quad
    \subfloat[d=500]{{\includegraphics[width=0.30\textwidth]{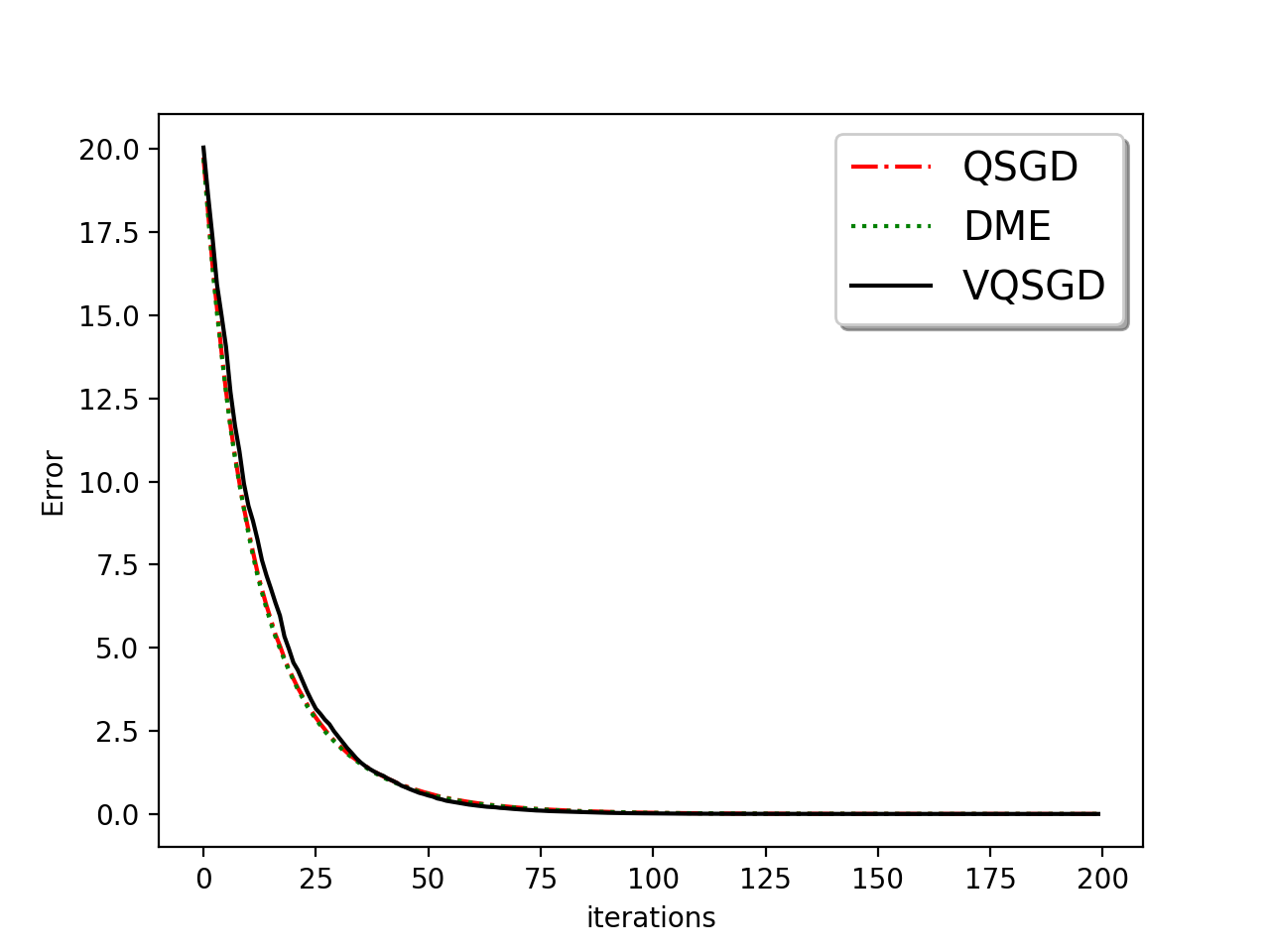} }}
    \caption{Comparison of convergence for the least square problem with $d=100, 200,500$.}
    \label{Compare}%
\end{figure*}

The results indicate that vqSGD achieves the same rate of convergence and accuracy as DME and QSGD while communicating only $\log(2d)$ bits and one real ($l_2$ norm of the vector form each server), whereas, DME (one bit stochastic quantization) and QSGD both require communication of about $\sqrt{d}$ bits and one real.

\begin{figure}[h!]
  \begin{center}
    \includegraphics[width=0.45\textwidth]{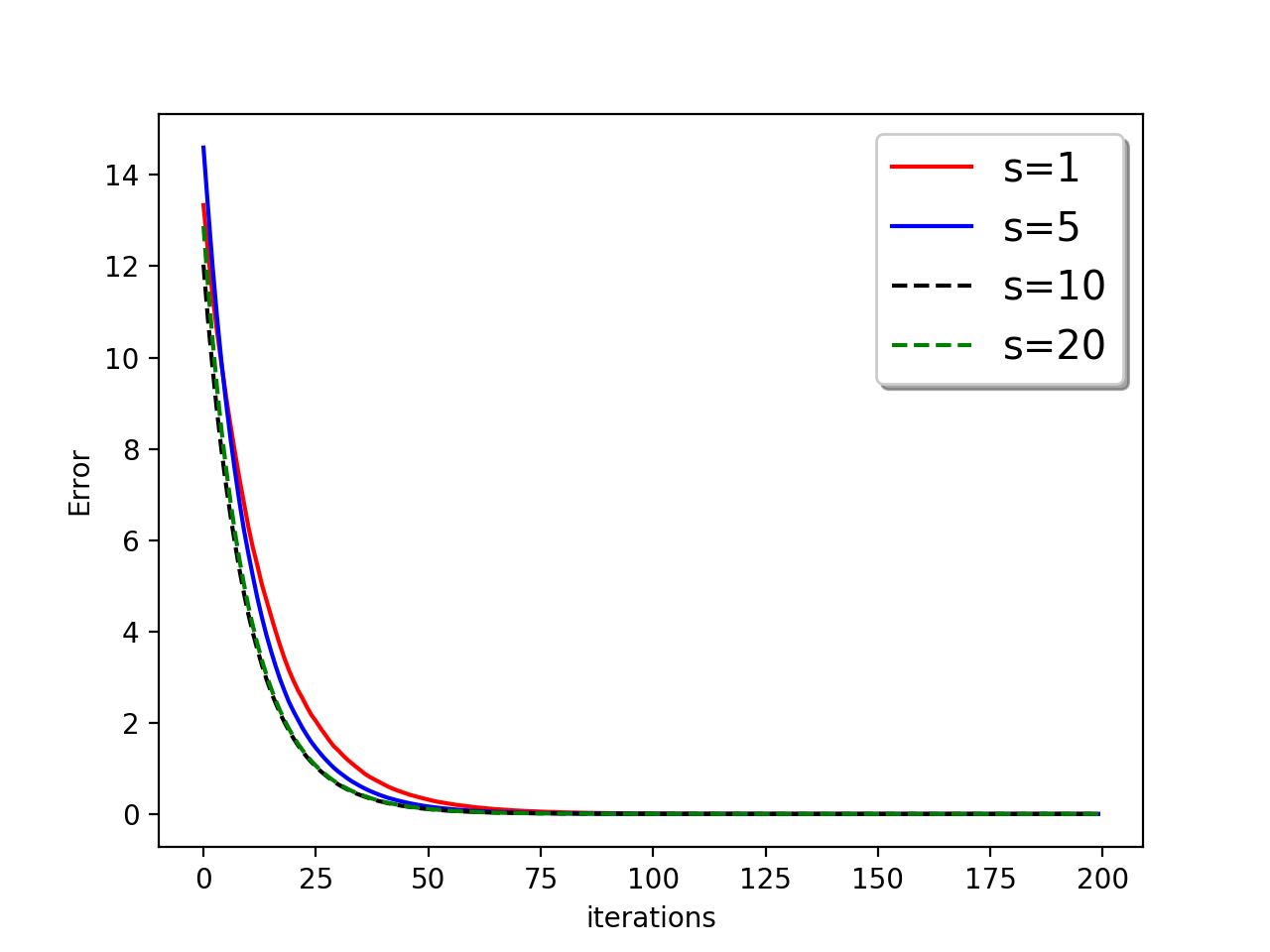}
    \caption{Convergence of $\p{\theta_t}$ for $s=1,5,10,20$.} for least square problem
    \label{compsd}
    \end{center}
\end{figure}
For the same problem setup, we also show the improvement in the performance of vqSGD using the repetition technique for variance reduction.  Recall that using repetition technique, each worker now sends $s$ different indices instead of $1$ which increases the communication to $s \log(2d)$ bits and $1$ real. 
In Figure~\ref{compsd} we plot the convergence of the lease square problem with $d=200$ with different values of $s = 1, 5,10,20$. We see the evident improvement in the  convergence of vqSGD using this repetition scheme with increasing $s$.

\textbf{Binary Classification:} 
We compared  the performance of vqSGD against DME and QSGD
for the binary classification problem with logistic regression using various datasets from the UCI repository~\cite{chang2011libsvm}. The logistic regression objective is defined as 
\begin{align}
\frac{1}{n}\sum_{i=1}^n \log(1+\exp(-b_i\p{a_i}^T\p{\theta}) +\frac{1}{2n}\|\p{\theta}\|^2_2, \label{logreg}
\end{align}
where $\p{\theta} \in \mathbb{R}^d$ is the parameter, $\p{a_i} \in \mathbb{R}^d$ is the feature data and $b_i\in \{-1,+1\}$ is its corresponding label. 

We partition the data into 20 equal-sized batches, each assigned to a different worker node. 
We calculate the classification error for different (test) datasets after training the parameter in the distributed settings (same as described in least square problem). 
Results of the experiments are presented in Table \ref{tab:UCI}, where each entry is  averaged over $20$ different runs.  

 \begin{table*}[h!]
    \centering
    \begin{tabular}{|l|c|c|c|}
    \hline
    Method  & DME   &  QSGD & vqSGD   \\
    \hline 
  a5a $(d=122)$ &$ 0.238\pm 0.0003 $&$ 0.238\pm 0.0002$& $0.2368 \pm 0.0029$ \\
  \hline
  a9a $(d=123)$ &$0.234\pm 0.0003 $&$0.234\pm 0.00017$ &$ 0.234 \pm 0.0015$\\
  \hline
  gisset-scale $(d=5000)$&$0.0947 \pm 0.00384$&$0.10475 \pm 0.006$&$0.1480\pm 0.0174$\\
  \hline
  splice $(d=60)$&$0.467\pm 0.017$&$0.4505\pm 0.0352$&$0.16618\pm 0.0054$\\
  \hline 
    \end{tabular}
  \caption{ Comparison in classification error (mean$\pm$ standard deviation) for various UCI datasets}
    \label{tab:UCI}
\end{table*}

We note that for most datasets, with the exception of  gisset-scale, vqSGD with $O(N \log d)$ bits of communication per iteration performs equally well or sometimes even better than QSGD and DME with $O(Nd)$ bits of communication per iteration.

\section{Conclusion}
We propose a general framework of convex-hull based private vector quantization schemes for distributed SGD that can be instantiated with any point set satisfying certain properties. 
The communication, variance and privacy tradeoffs for these mechanisms depend on the choice of point set. 
The proposed cross-polytope quantization scheme with low communication overhead is shown experimentally to achieve convergence rates similar to the existing state-of-the-art quantization schemes which use orders of magnitudes more communication. While the explicit efficient schemes seems to have a $\log d$-factor communication overhead, we believe it will be hard but interesting to get rid of the factor with deterministic construction.

Information theoretically, we are asking the question of computing the variance of an unbiased estimator of points in the unit sphere, in terms of its unconditional entropy. We have established the exact trade-off between variance and entropy for almost surely bounded estimators. We, in this paper, have tried to minimize the communication: but we believe our techniques will be applicable to variance reduction techniques as well - at the expense of $\Omega(d)$ communication.

\bibliographystyle{plain}
\bibliography{ref}

\appendix

\end{document}